\documentclass[10pt,journal,compsoc]{IEEEtran}

\ifCLASSOPTIONcompsoc
  \usepackage[nocompress]{cite}
\else
  \usepackage{cite}
\fi

\ifCLASSINFOpdf
  \usepackage[pdftex]{graphicx}
\else
  \usepackage[dvips]{graphicx}
\fi
    
\ifCLASSOPTIONcompsoc
 \usepackage[caption=false,font=footnotesize,labelfont=sf,textfont=sf]{subfig}
\else
 \usepackage[caption=false,font=footnotesize]{subfig}
\fi

\usepackage{amsmath,amsfonts,amsthm,amssymb}
\usepackage{algorithm, algorithmic}
\usepackage{array}
\usepackage[colorlinks]{hyperref} 
\usepackage{enumitem}
\usepackage{booktabs,multirow,multicol}
\usepackage[table]{xcolor}

\allowdisplaybreaks

\theoremstyle{plain}
\newtheorem{theorem}{Theorem}
\newtheorem{assumption}{Assumption}
\newtheorem{proposition}[theorem]{Proposition}
\newtheorem{lemma}[theorem]{Lemma}
\newtheorem{corollary}[theorem]{Corollary}

\theoremstyle{definition}

\theoremstyle{remark}
\newtheorem{remark}{Remark}

\DeclareMathOperator*{\argmin}{arg\,min}
\DeclareMathOperator{\var}{Var}
\newcommand{\bftheta}{\boldsymbol{\theta}}
\newcommand{\bfphi}{\boldsymbol{\phi}}
\newcommand{\bfa}{\mathbf{a}}
\newcommand{\bfb}{\mathbf{b}}
\newcommand{\bfg}{\mathbf{g}}
\newcommand{\bfG}{\mathbf{G}}
\newcommand{\bfh}{\mathbf{h}}
\newcommand{\bfH}{\mathbf{H}}
\newcommand{\bfI}{\mathbf{I}}
\newcommand{\bfM}{\mathbf{M}}
\newcommand{\bfP}{\mathbf{P}}

\newcommand{\bfv}{\mathbf{v}}

\newcommand{\bfW}{\mathbf{W}}

\newcommand{\bfX}{\mathbf{X}}

\newcommand{\bfz}{\mathbf{z}}
\newcommand{\data}[1]{\mathcal{D}^\mathrm{#1}}
\newcommand{\loss}[1]{\ell^\mathrm{#1}}
\newcommand{\metaloss}{\mathcal{L}}
\newcommand{\reg}{r}
\newcommand{\expect}{\mathbb{E}}
\newcommand{\real}{\mathbb{R}}
\newcommand{\batch}{\mathcal{B}}
\newcommand{\lin}{\text{lin}}
\newcommand{\lightcell}{\cellcolor{gray!25}}
\newcommand{\deepcell}{\cellcolor{gray!75}}

\begin{document}

\title{Learnable Loss Geometries with Mirror Descent for Scalable and Convergent Meta-Learning}

\author{Yilang Zhang, Bingcong Li, and Georgios B. Giannakis,~\IEEEmembership{Fellow,~IEEE}
\thanks{Y. Zhang and G. B. Giannakis are with the Dept. of Electrical and Computer
Engineering, University of Minnesota, Minneapolis, MN 55455, USA. B. Li is with Dept. of Computer Science, ETH Z{\"u}rich, 8092 Z{\"u}rich, Switzerland. Emails: \{zhan7453, georgios\}@umn.edu and bingcong.li@inf.ethz.ch.}
}

\markboth{IEEE TRANSACTIONS ON PATTERN ANALYSIS AND MACHINE INTELLIGENCE (SUBMITTED)}{}


\IEEEtitleabstractindextext{
\begin{abstract}
Utilizing task-invariant knowledge acquired from related tasks as prior information, meta-learning offers a principled approach to learning a new task with limited data records. Sample-efficient adaptation of this prior information is a major challenge facing meta-learning, and plays an important role because it facilitates training the sought task-specific model with just \textit{a few} optimization steps. Past works deal with this challenge through preconditioning that speeds up convergence of the per-task training. Though effective in representing locally quadratic loss curvatures, simple linear preconditioning can be hardly potent with complex loss geometries. Instead of relying on a quadratic distance metric, the present contribution copes with complex loss metrics by learning a versatile distance-generating function, which induces a nonlinear mirror map to effectively capture and optimize a wide range of loss geometries. With suitable parameterization, this generating function is effected by an expressive neural network that is provably a valid distance. Analytical results establish convergence of not only the proposed method, but also all meta-learning approaches based on preconditioning. To attain gradient norm less than $\epsilon$, the convergence rate of $\mathcal{O}(\epsilon^{-2})$ is on par with standard gradient-based meta-learning methods. Numerical tests on few-shot learning datasets demonstrate the superior empirical performance of the novel algorithm, as well as its rapid per-task convergence, which markedly reduces the number of adaptation steps, hence also accommodating large-scale meta-learning models. 
\end{abstract}

}

\maketitle

\IEEEraisesectionheading{\section{Introduction}\label{sec:introduction}}
\IEEEPARstart{T}{he} advent and advances of deep learning (DL) have led to documented success across a broad spectrum of fields, including computer vision~\cite{CNN} and nature language processing~\cite{transformer}. However, DL relies heavily on large-scale and high-capacity models, which necessitate extensive training using a vast number of labeled data. However, the data collection and annotation process in certain applications can be non-trivial and costly, requiring substantial human efforts and prohibitively high budget. As an example, the ResNet-50 neural network (NN) model, which is widely adopted in natural and medical image processing, contains over 23 million parameters~\cite{ResNet}. In contrast, a medical image dataset such as BreastMNIST consists of merely 780 data~\cite{BreastMNIST}, given the imperative to uphold medical confidentiality by avoiding the disclosure of private patient information. Consequently, large models tend to overfit the given dataset, and lead to degraded generalization ability. This calls for a paradigm shift to reach the desired model capacity with manageable number of labeled data. 

Interestingly, in comparison to DL models, humans can perform exceptionally well on complicated tasks such as object recognition or concept comprehension with minimal samples. For instance, a child can efficiently learn to recognize objects such as cars and planes after observing solely a couple of pictures~\cite{Meta-LSTM}. How to integrate this data-efficient learning ability of humans into DL is both appealing and crucial, especially for applications with limited data records. Examples of such applications include drug molecule discovery~\cite{metaapp-drug}, minority language translation~\cite{metaapp-trans}, and catastrophic weather prediction~\cite{metaapp-weather-pred}. 

To cross-pollinate learning of humans to DL, \emph{meta-learning} (aka \emph{learning-to-learn}) has been recognized to offer a powerful framework for learning an unseen task from limited labeled data. Specifically, meta-learning seeks to accumulate \emph{task-invariant prior} information from a collection of interrelated tasks, that can subsequently aid the learning of a novel, albeit related task. Although limited data may be available to learn this new sought task, the prior serves as a strong inductive bias that effectively transfers knowledge to aid its learning~\cite{MAML}. In image classification for instance, a feature extractor learned from a collection of given tasks can act as a common prior, thus benefiting a variety of other image classification tasks. 

Depending on how this ``data-limited learning'' is performed per task, existing meta-learning algorithms can be categorized as either NN- or optimization-based ones. In the former, the learning of a task is modeled by an NN mapping from its training data to task-specific model parameters~\cite{mem-aug,Meta-LSTM, neural-attentive}. The prior information is encoded in the NN weights, which are shared and optimized across all tasks. With the effectiveness of NNs to approximate complex mappings granted, their black-box structure challenges their reliability and interpretability. In contrast, optimization-based meta-learning alternatives interpret ``data-limited learning'' as a cascade of a few optimization iterations (aka \emph{adaptation}) over the model parameters. The prior here is captured by the shared hyperparameters of the iterative optimizer. A representative algorithm is model-agnostic meta-learning (MAML)~\cite{MAML}, which views the prior as a learnable task-invariant initialization of the optimizer. By starting from an informative initial point, the model parameters can rapidly converge to a stationary point within a few gradient descent (GD) steps. It has been pointed out that the initialization in MAML can be viewed as the mean of a multivariate Gaussian prior over the model parameters under a second-order Taylor approximation~\cite{LLAMA}. Building upon MAML, a series of variants have been proposed to learn different priors~\cite{iMAML,R2D2,MetaOptNet,BMAML,MetaProxNet}. 

While optimization-based meta-learning has been proven effective numerically, recent studies suggest that its stability heavily rely on the convergence of per-task optimization~\cite{iMAML, iBaML}. Hence, a straightforward improvement is to grow the number of descent iterations. However, this can be infeasible as the overall complexity of meta-learning scales linearly with the number of adaptation GD steps~\cite{iMAML}. Moreover, using accelerated first-order optimizers, such as Nesterov's algorithm~\cite{Nesterov}, introduces extra backpropagation complexity when optimizing the prior. To improve the per-task convergence without markedly adding to the computational overhead, another line of research focuses on second-order optimization using a learnable preconditioning matrix~\cite{MetaSGD,MetaMD,MetaCurvature,WarpGrad,sparse-MAML}. Essentially, this preconditioner captures the quadratic curvature of the training loss function, and linearly transforms the gradient based on the curvature. To acquire more expressive and generic forms of preconditioning, recent advances suggest replacing the linear matrix multiplication with a nonlinear NN transformation~\cite{MetaKFO}, or, a singular value preconditioning~\cite{GAP}. 

Unlike previous works that rely on preconditioning to model the prior of \emph{quadratic loss curvature}, this contribution proposes to learn a distance metric that captures a  \emph{generic loss geometry} prior shared across tasks, thus accelerating the convergence of task adaptation. This generalization is effected by replacing the preconditioned GD (PGD) with the more generic mirror descent algorithm (MiDA)~\cite{MDA}. We thus term the resultant meta-learning algorithm MetaMiDA. All in all, our contribution is threefold. 
\begin{enumerate}
\item [C1.] An NN-based distance-generating function is developed to model generic loss geometries, with theoretical guarantees on the validity of the resultant mirror map. 
\item [C2.] Convergence analysis of not only MetaMiDA, but also all PGD-based meta-learning methods is established. The rate $\mathcal{O}(\epsilon^{-2})$ matches GD-based meta-learning~\cite{converge-onestep-MAML, converge-multistep-MAML, converge-partial-param}, yet relying on fewer assumptions. 
\item [C3.] Extensive numerical tests showcase MetaMiDA's improved empirical performance and accelerated adaptation, even with merely \emph{one} optimization step. 
\end{enumerate}

This work extends our previous conference paper~\cite{MetaMDA-ICASSP} in three key aspects. First, while the mirror map in~\cite{MetaMDA-ICASSP} is restricted to have a triangular Jacobian, this contribution allows for more generic mirror maps by modeling alternatively the distance-generating function. Second, theoretical analysis is provided to guarantee the convergence of a broad family of meta-learning algorithms. Lastly, numerical experiments are vastly expanded to incorporate extensive datasets, large-scale models, complexity analysis, and more challenging setups such as cross-domain generalization. 

\textbf{Notation.} Bold lowercase (capital) letters denote column vectors (matrices); $\langle \cdot, \cdot \rangle$ and $\cdot^\top$ represent respectively inner-product and transposition; and $\nabla_i$ stands for partial derivative wrt the $i$-th function argument (input or parameter). 

\section{Preliminaries}
This section outlines the meta-learning formulation, followed by a recap of popular meta-learning approaches, and their limitations, especially that of scalability. 

\subsection{Problem setup}
To enable learning a new task from limited data, meta-learning extracts task-invariant information from a collection of $T$ given tasks. Let $t \in \{ 1,\ldots,T \}$ be a uniform discrete random variable indexing these tasks; that is, $\Pr(t=1) = \ldots = \Pr(t=T) = {T}^{-1}$. Each task comprises a dataset $\data{}_t := \{ (\mathbf{x}_t^n, y_t^n) \}_{n=1}^{N_t}$ of $N_t$ (data, label) pairs that are split into a training subset $\data{trn}_t \subset \data{}_t$, and a disjoint validation subset $\data{val}_t := \data{}_t \setminus \data{trn}_t$. In addition to the $T$ given tasks, a new task indexed by $\star$, contains a small training subset $\data{trn}_\star$, and a set of test data $\{ \mathbf{x}_\star^n \}_{n=1}^{N_\star^{\mathrm{tst}}}$ for which the corresponding labels $\{ y_\star^n \}_{n=1}^{N_\star^{\mathrm{tst}}}$ are to be predicted. 

The key premise of meta-learning is that the aforementioned $T$ tasks share related model structures or data distributions. Thus, one can postulate a large model shared across all tasks, along with distinct model parameters $\bfphi_t \in \real^d$ pertaining to each individual task $t$. Given that the cardinality $N_t^\mathrm{trn} := |\data{trn}_t|$ can be much smaller than $d$, learning a task by directly optimizing $\bfphi_t$ over $\data{trn}_t$ could lead to severe overfitting, and it is thus undesirable. Fortunately, since $\sum_{t=1}^T N_t^\mathrm{val}$ can be considerably large, a task-invariant prior can be learned using $\{ \data{val}_t \}_{t=1}^T$ to render per-task learning well posed. Once acquired, this prior can be readily transferred to the new task $\star$ to facilitate its training on $\data{trn}_\star$. 

Letting $\bftheta \in \real^D$ denote the parameter of the prior (aka meta-parameter), the meta-learning objective can be formulated as a bilevel optimization problem. The inner-level (task-level) trains each task-specific model by optimizing $\bfphi_t$ using $\data{trn}_t$ and $\bftheta$ provided by the outer-level (meta-level). The outer-level adjusts $\bftheta$ by evaluating the optimized $\{ \bfphi_t \}_{t=1}^T$ on $\{ \data{val}_t \}_{t=1}^T$. The two levels depend on each other and yield the following nested bilevel objective
\begin{subequations}
\label{eq:metalearning-obj}
\begin{align}
\label{eq:metalearning-obj-meta}
	&\bftheta^* = \argmin_{\bftheta} \; \expect_t \loss{val}_t(\bfphi_t^* (\bftheta)) := \frac{1}{T} \sum_{t=1}^T \loss{val}_t(\bfphi_t^* (\bftheta)) \\
\label{eq:metalearning-obj-task}
\hspace*{-0.3cm} \text{s.t.}\;&\;\bfphi_t^* (\bftheta) = \argmin_{\bfphi_t} \loss{trn}_t(\bfphi_t) + \reg (\bfphi_t; \bftheta),\;\;  t=1,\ldots,T
\end{align}
\end{subequations}
where $\loss{val}_t$ ($\loss{trn}_t$) is the validation (training) loss function, 
and $\reg$ is the regularizer accounting for the task-invariant prior parameterized by $\bftheta$. With $\mathrm{set}$ denoting either $\mathrm{trn}$ or $\mathrm{val}$, we can view $\loss{set}_t (\bfphi_t)$ and $\reg(\bfphi_t; \bftheta)$ as the negative log-likelihood (NLL) $-\log p(\mathbf{y}_t^{\mathrm{set}} | \bfphi_t; \mathbf{X}_t^{\mathrm{set}})$, and the negative log-prior (NLP) $-\log p(\bfphi_t ; \bftheta)$. Here, matrix $\bfX_t^\mathrm{set}$ collects all the data in $\data{set}_t$, and $\mathbf{y}_t^\mathrm{set}$ is the corresponding label vector. Bayes' rule then implies $\bfphi_t^* (\bftheta) = \argmin_{\bfphi_t} - \log p(\bfphi_t | $ $\mathbf{y}_t^{\mathrm{trn}}$; $\mathbf{X}_t^{\mathrm{trn}}, \bftheta)$ is the maximum a posteriori (MAP) estimator. 

\subsection{Past works on meta-learning}
Unfortunately, reaching the global optimum $\bfphi_t^*$ is generally infeasible because the task-specific model can be highly nonlinear wrt $\bfphi_t$. Hence, a prudent remedy is to rely on an approximate solver generated from a tractable optimizer. Depending on  how this solver is designed, meta-learning approaches can be grouped into NN- and optimization-based ones. Methods in the first group rely on an NN optimizer $\hat{\bfphi}_t (\bftheta) = \mathrm{NN}(\data{trn}_t; \bftheta) \approx \bfphi_t^* (\bftheta)$ to model the map from $\data{trn}_t$ to $\bfphi_t^*$, with the sought prior information captured by the NN learnable weights collected in $\bftheta$~\cite{mem-aug,Meta-LSTM,neural-attentive}. Although NNs offer universal approximators for a large family of functions~\cite{univ-approx}, the black box NN structure challenges the interpretability of $\bftheta$. 

To enhance interpretability and robustness of per-task training, optimization-based meta-learning resorts to an iterative optimizer to approximately solve~\eqref{eq:metalearning-obj-task}, where the prior $\bftheta$ is formed by hyperparameters of the optimizer. A representative algorithm is model-agnostic meta-learning (MAML)~\cite{MAML}.  MAML replaces~\eqref{eq:metalearning-obj-task} with a $K$-step GD minimizing the NLL, and forms $\bftheta$ as a common initialization shared across tasks; i.e., 
\begin{subequations}
\label{eq:MAML-obj}
\begin{align}
\label{eq:MAML-obj-meta}
	\bftheta^* &= \argmin_{\bftheta} \;\expect_t \loss{val}_t(\bfphi_t^K (\bftheta)) \\
\label{eq:MAML-obj-task}
	\text{s.t.}\;\;  \bfphi_t^0 (\bftheta) &= \bftheta, \;\; t=1,\ldots,T, \\
	\bfphi_t^{k+1} (\bftheta) &= \bfphi_t^k (\bftheta) - \alpha \nabla \loss{trn}_t(\bfphi_t^k (\bftheta)),\;k=0,\ldots,K-1 \nonumber
\end{align}
\end{subequations}
where $\alpha > 0$ is the training step size. Although MAML intentionally sets $\reg (\bfphi_t; \bftheta) = 0$, it has been shown that under a second-order Taylor approximation, MAML satisfies~\cite{LLAMA}
\begin{equation*}
	\bfphi_t^K (\bftheta) \approx \bfphi_t^* (\bftheta) = \argmin_{\bfphi_t} \loss{}_t (\bfphi_t) + \frac{1}{2} \| \bfphi_t - \bftheta \|_{\mathbf{\Lambda}_t}^2, ~\forall t
\end{equation*}
where the precision matrix $\mathbf{\Lambda}_t$ is determined by $\nabla^2 \loss{}_t (\bftheta)$, $\alpha$, and $K$. This indicates that MAML's optimization strategy~\eqref{eq:MAML-obj-task} is approximately tantamount to an implicit Gaussian prior $p(\bfphi_t; \bftheta) = \mathcal{N} (\bftheta, \mathbf{\Lambda}_t^{-1})$, with the task-invariant initialization $\bftheta$ serving as the mean vector. Alongside implicit priors, their explicit counterparts have also been investigated with various distributions including isotropic Gaussian~\cite{iMAML}, diagonal Gaussian~\cite{ABML}, Laplacian~\cite{meta-prune}, partially degenerate~\cite{R2D2, MetaOptNet}, and data-driven ones~\cite{BMAML, MetaProxNet}. For example,~\cite{iMAML} chooses $r(\bfphi_t; \bftheta) = -\log \mathcal{N} (\bftheta, \lambda^{-1} \bfI_d) = (\lambda / 2) \| \bfphi_t - \bftheta \|_2^2$, and performs task-level optimization via 
\begin{equation}
\label{eq:explicit-prior}
	\bfphi_t^{k+1} (\bftheta) = \bfphi_t^k (\bftheta) - \alpha [\nabla \loss{trn}_t(\bfphi_t^k (\bftheta)) + \nabla_1 r (\bfphi_t^k; \bftheta)],\, \forall k. 
\end{equation}

\subsection{Key challenge in meta-learning}
It is known that GD converges with sublinear rate $\mathcal{O}(1/K)$, when $\loss{trn}_t (\bfphi_t) + r (\bfphi_t; \bftheta)$ is Lipschitz-smooth wrt $\bfphi_t$~\cite{nonlinear-prog}. This necessitates a large $K$ for $\bfphi_t^K$ to approach a stationary point, which is globally optimal if the training objective is also assumed convex. Further, it has been shown that the gradient error between~\eqref{eq:metalearning-obj-meta} and~\eqref{eq:MAML-obj-meta} grows linearly with the convergence error $\expect_t \| \bfphi_t^K - \bfphi_t^* \|_2$~\cite{iMAML, iBaML}. This promotes a sufficiently large $K$ to ensure that~\eqref{eq:MAML-obj} approximates well~\eqref{eq:metalearning-obj}. 

Nevertheless, the overall computation for solving~\eqref{eq:MAML-obj-meta} grows linearly with $K$~\cite{MAML} that can prohibitively increase complexity. Although accelerated optimizers such as Nesterov's algorithm~\cite{Nesterov} can improve the convergence rate of~\eqref{eq:MAML-obj-task} to $\mathcal{O}(1/K^2)$, the constant hidden inside $\mathcal{O}$ could be considerably large, and the Nesterov's momentum would introduce extra backpropagation and thus markedly grow the burden for computing the gradient in~\eqref{eq:MAML-obj-meta}. As a consequence, attention has been placed towards preconditioned (P) GD solvers, as in the following task-level update\footnote{The PGD update is presented with an implicit prior for simplicity, while this approach can be also combined with explicit priors~\eqref{eq:explicit-prior}.}
\begin{equation}
\label{eq:PGD}
	\bfphi_t^{k+1} (\bftheta) = \bfphi_t^k (\bftheta) - \alpha \mathbf{P} (\bftheta_P) \nabla \loss{trn}_t(\bfphi_t^k (\bftheta)),\, \forall k. 
\end{equation}
where $\bfphi_t^0 = \bftheta_\phi,\,\forall t$ is the shared initialization, $\bftheta_P$ parameterizes $\bfP \in \real^{d\times d}$, and $\bftheta := \{ \bftheta_\phi, \bftheta_P \}$. To ensure~\eqref{eq:PGD} incurs affordable complexity, the preconditioner $\mathbf{P}$ must be simple enough so that $\mathbf{P} (\bftheta_P) \nabla \loss{trn}_t (\bfphi_t^k)$ has computational complexity $\mathcal{O}(d)$. Examples of simple preconditioners include diagonal~\cite{MetaSGD, MetaMD, sparse-MAML}, block-diagonal~\cite{MetaCurvature, MT-net}, and sparse~\cite{sparse-MAML} matrices. One straightforward variant is to rely on a matrix $\bfP_t^k (\bfv_t^k; \bftheta_P)$ evolving with the iterations, where $\bfv_t^k$ is a vector encoding the ``context'' of task $t$ at step $k$~\cite{ModGrad,PAMELA}. A more generic preconditioning can be formed by replacing the linear transformation $\mathbf{P} (\bftheta_P) \nabla \loss{trn}_t (\bfphi_t^k)$ with an additional nonlinear NN $f_P (\nabla \loss{trn}_t (\bfphi_t^k); \bftheta_P)$~\cite{MetaKFO}, but unfortunately convergence of the resultant iteration may not be guaranteed. In addition,~\cite{GAP} advocates learning a Riemannian metric by preconditioning the singular values of the learning model's weights. However, this method requires burdensome singular value decomposition (SVD) that is infeasible for large-scale NNs. To enhance the scalability, the singular value preconditioning is further approximated by a diagonal linear preconditioning matrix~\cite{GAP}. 

\begin{figure}[!t]
\begin{center}
	\subfloat[]{\includegraphics[width=.49\columnwidth]{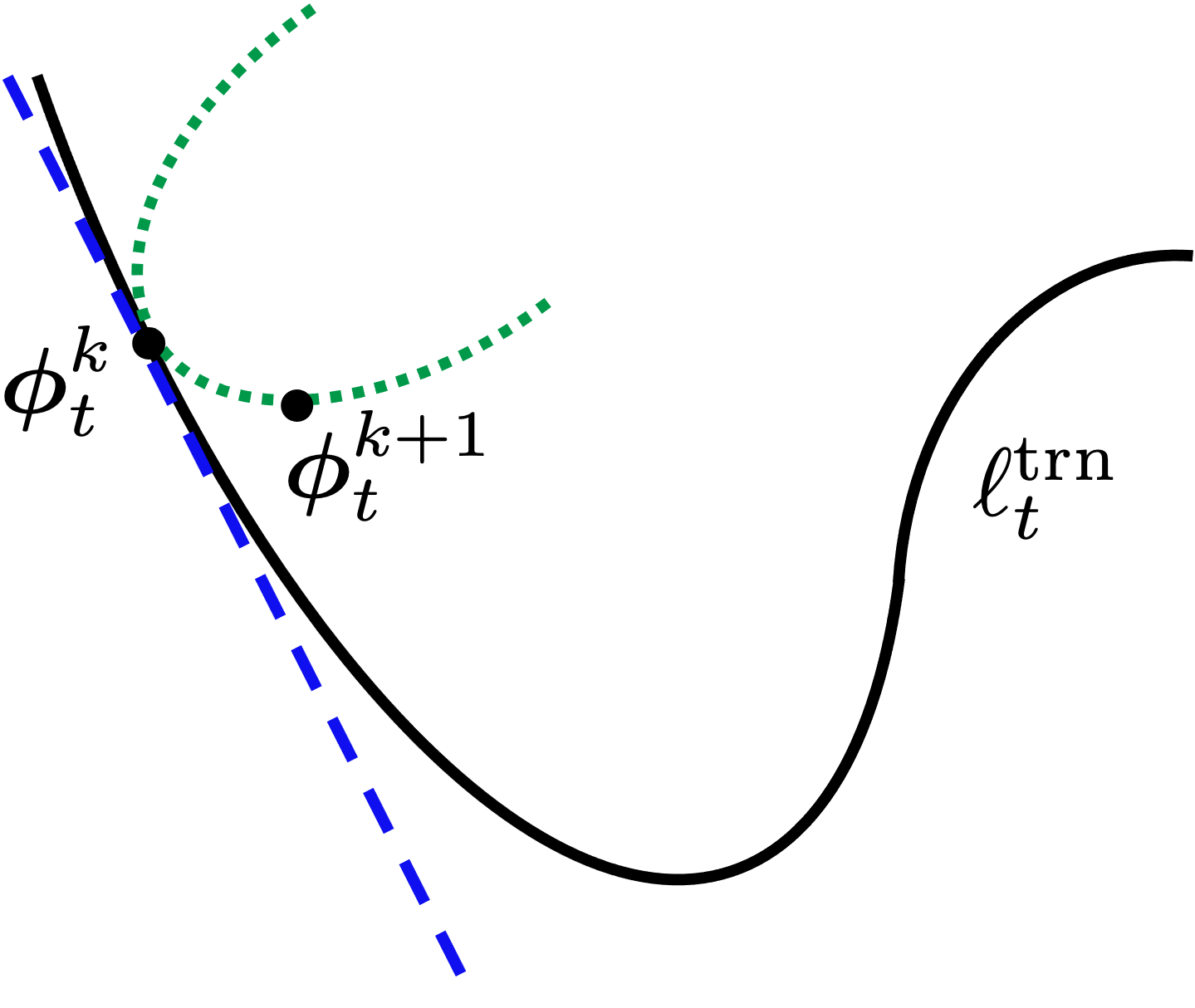}}
	\hfil
	\subfloat[]{\includegraphics[width=.44\columnwidth]{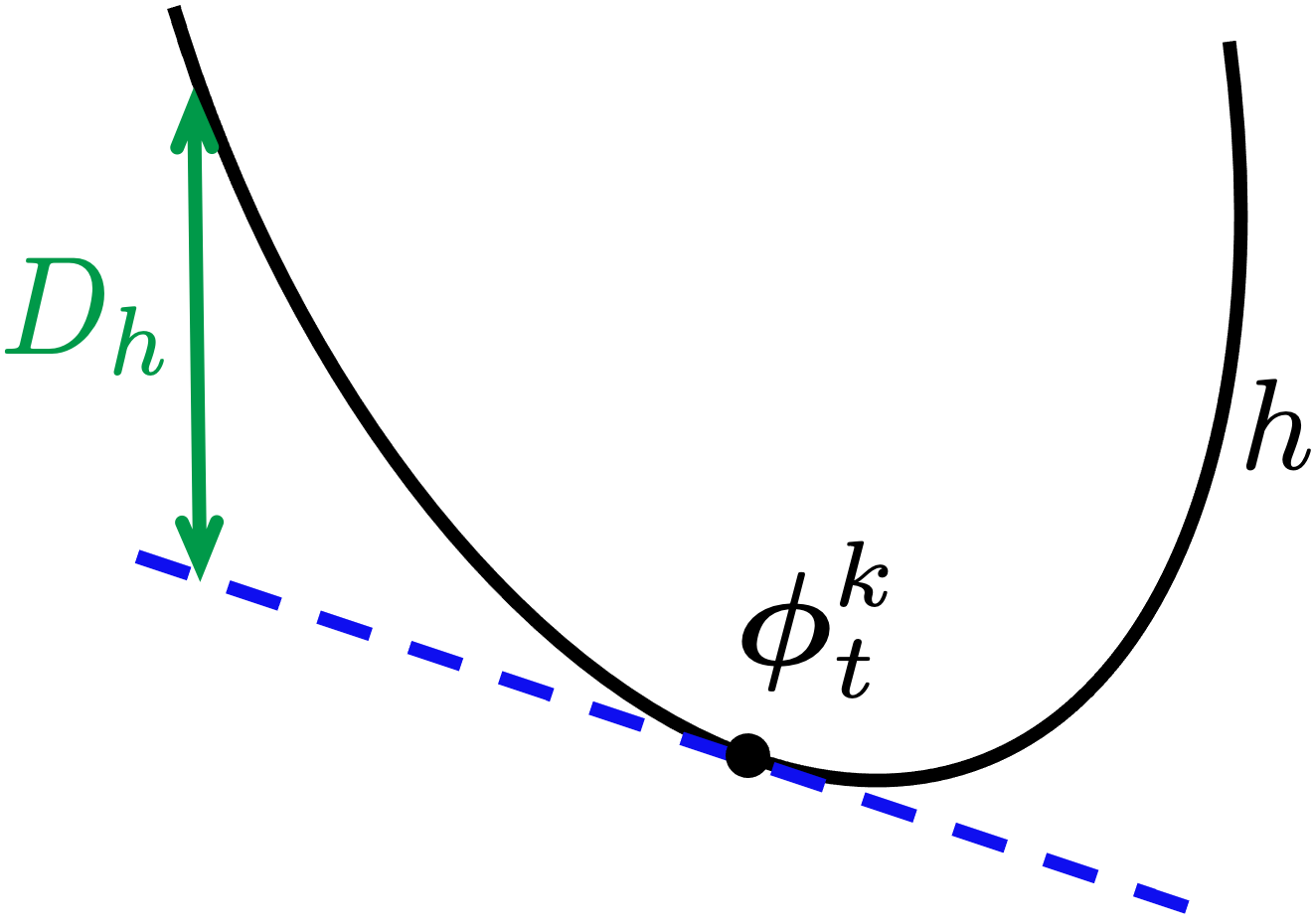}}
\end{center}
\vskip -0.1in
\caption{Illustration of (a) (P)GD, and (b) DGF. The blue dashed line and green dotted lines stand for linearization and quadratic upper bounds.}
\label{fig:PGD-DGF}
\end{figure}

Essentially, GD conducts a pre-step greedy search with a quadratic approximation of the training loss. To see this, letting 
$\lin(\loss{trn}_t,\tilde{\bfphi}_t)(\bfphi_t) := \loss{trn}_t (\tilde{\bfphi}_t) + (\bfphi_t - \tilde{\bfphi}_t)^\top \nabla \loss{trn}_t (\tilde{\bfphi}_t)$ be the linearization of $\loss{trn}_t$ at $\tilde{\bfphi}_t \in \real^d$,~\eqref{eq:MAML-obj-task} reduces to
\begin{equation}
\label{eq:GD-greedy}
	\hspace{-0.05cm}\bfphi_t^{k+1} = \argmin_{\bfphi_t} \lin(\loss{trn}_t, \bfphi_t^k) (\bfphi_t) + \frac{1}{2\alpha} \| \bfphi_t - \bfphi_t^k \|_2^2, \;\forall k
\end{equation}
where dependencies on $\bftheta$ are dropped hereafter for notational brevity; see also Figure~\ref{fig:PGD-DGF}(a) for an illustration. The term $\frac{1}{2\alpha} \| \bfphi_t - \bfphi_t^k \|_2^2$ implies the curvature approximation $\nabla^2 \loss{}_t (\bfphi_t^k) \approx \frac{1}{\alpha} \mathbf{I}_d$, while~\eqref{eq:PGD} refines this isotropic term to a more general quadratic term $\| \bfphi_t - \bfphi_t^k \|_{\bfP}^2 := \frac{1}{2\alpha} (\bfphi_t - \bfphi_t^k)^\top \mathbf{P}^{-1} (\bfphi_t - \bfphi_t^k)$. It is worth noting that when $\bfP$ is singular, one can replace $\bfP^{-1}$ with the pseudo-inverse $\bfP^\dagger$, and also expand $\argmin$ in~\eqref{eq:GD-greedy} to a set containing $\bfphi_t^{k+1}$. 

Since the quadratic term is defined to be symmetric wrt $\bfphi_t^k$, and has a constant Hessian matrix, the approximation can be significantly biased on regions where $\| \nabla^2 \loss{trn}_t (\bfphi_t) - \bfP^{-1} \|_2$ is large. Thus, (P)GD is particularly effective when $K$ is large and $\alpha$ is small, which gradually ameliorates $\bfphi_t^k$ to a stationary point. In meta-learning however, the standard setup requires a small $K$ (e.g., $1$ or $5$) and a sufficiently large $\alpha$, so that the model can quickly adapt to the task with low complexity. This discrepancy highlights the need for learning more expressive loss geometries. 

\section{Learning loss geometries with MetaMiDA}
Instead of relying on quadratic approximations of the local loss induced by certain norms (e.g., $\| \cdot \|_2$ and $\| \cdot \|_{\mathbf{P}}$), our fresh idea is to learn a data-driven distance metric that captures a broader spectrum of loss geometries. This will be accomplished by learning the so-termed distance generating function (DGF), which is introduced first. All the proofs in this section are delegated to the Appendices. 

\subsection{Distance metric reflecting loss geometry prior}
Leveraging the idea of mirror descent~\cite{MDA}, one can replace $\frac{1}{2} \| \bfphi_t - \bfphi_t^k \|_2^2$ in~\eqref{eq:GD-greedy} with a generic metric $D_h$ to arrive at
\begin{equation}
\label{eq:MDA-greedy}
	\bfphi_t^{k+1} = \argmin_{\bfphi_t} \lin(\loss{trn}_t, \bfphi_t^k) (\bfphi_t) + \frac{1}{\alpha} D_h (\bfphi_t, \bfphi_t^k),\;\forall k
\end{equation}
where $D_h (\bfphi_t, \bfphi_t^k) := h(\bfphi_t) - \lin (h(\bfphi_t), \bfphi_t^k)$ is the Bregman divergence depicted in Figure~\ref{fig:PGD-DGF}(b), and the associated DGF $h: \real^d \mapsto \real$ is strongly convex to ensure the existence and uniqueness of the minimizer. Then, applying the stationary point condition leads to the mirror descent update
\begin{equation}
\label{eq:MDA}
	\bfphi_t^{k+1} = \nabla h^* \big( \nabla h(\bfphi_t^k) - \alpha \nabla \loss{trn}_t (\bfphi_t^k) \big),\;\forall k
\end{equation}
where $h^* (\mathbf{z}) := \sup_{\bfphi} \bfphi^\top \mathbf{z} - h(\bfphi)$ is the Fenchel conjugate of $h$. Since $h$ is strongly convex, it holds that $h^*: \real^d \mapsto \real$ is convex and Lipschitz-smooth~\cite{MDA}. As a special case, with  $h(\cdot) = \frac{1}{2}\| \cdot \|_2^2$, it is easy to verify that~\eqref{eq:MDA} boils down to~\eqref{eq:GD-greedy} due to the self-duality of $\| \cdot \|_2$. Likewise,~\eqref{eq:PGD} can be recovered with $h(\cdot) = \frac{1}{2}\| \cdot \|_{\mathbf{P}}^2$, where $\nabla h$ reduces to a linear mapping. 

Function $h$ reflects our prior knowledge about the geometry of $\loss{trn}_t$. Upon  setting $h = \loss{trn}_t$ (though $\loss{trn}_t$ may not be strongly convex) in~\eqref{eq:MDA-greedy} gives $\bfphi_t^{k+1} = \argmin_{\bfphi_t} \loss{trn}_t (\bfphi_t)$, which recovers the original NLL minimization solved in~\eqref{eq:MAML-obj-task}. Thus, an ideal choice of $h$ would yield $h \approx \loss{trn}_t$ (up to a constant) within a sufficiently large region around $\bfphi_t^{k}$. If $h$ is properly chosen for the geometry of the optimization landscape, mirror descent can outperform GD markedly in either convergence rate or constant factor~\cite{MD-cvx}. 

Different from past works with \emph{preselected} $h$ (e.g., $\| \cdot \|_2$ and $\| \cdot \|_{\mathbf{P}}$), we are here after a \emph{data-driven} $h$ that best fits the given tasks. Toward this end, we reformulate~\eqref{eq:MDA} to yield an update of the dual vector $\bfz_t := \nabla h(\bfphi_t)$ as
\begin{equation}
\label{eq:MDA-dual}
	\bfz_t^{k+1} = \bfz_t^k - \alpha \nabla \loss{trn}_t ( \nabla h^* (\bfz_t^k) ), \;\forall k
\end{equation}
where the primal vector is recovered via $\bfphi_t^K = \nabla h^* (\mathbf{z}_t^K)$. It can be observed that the dual update~\eqref{eq:MDA-dual} removes the explicit computation of $\nabla h$, requiring merely the calculation of $\nabla h^*$. Hence, instead of learning $\bftheta_\phi$ and $h$ in the primal space, we learn the dual initialization $\bfz_t^0 = \bftheta_z := \nabla h(\bftheta_\phi),\,\forall t$ along with the conjugate $h^*$. 

\subsection{Modeling versatile \texorpdfstring{$h^*$}{h*} via an expressive NN}
\label{sec:h*}
To ensure that the data-driven $h^*$ yields a strongly convex $h$, it suffices to maintain its convexity and Lipschitz-smoothness. The following theorem asserts that such an $h^*$ can be formed by slightly modifying a multi-layer NN. 

\begin{theorem}[Convex and Lipschitz-smooth NN mapping]
\label{thm:cvx-Lip-NN}
Let $h^*(\mathbf{z}; \bftheta_h)$ be an $I$-layer NN, with per-layer neurons $\{ \bfa_i \}_{i=1}^I$, and parameter $\bftheta_h := \{ \bfW_i, \bfM_i, \bfb_i \}_{i=1}^I$ satisfying
\begin{subequations}
\label{eq:NN-h*}
\begin{align}
\label{eq:NN-h*-forward}
	\bfa_i &= \sigma(\bfW_i^\top \bfa_{i-1} + \bfM_i^\top \bfz + \bfb_i),\;i=1,\ldots,I, \\
\label{eq:NN-h*-inout}
	\bfa_0 &= \mathbf{z}, \;h^*(\bfz; \bftheta_h) := \bfa_I. 
\end{align}	
\end{subequations}
Then $h^*(\mathbf{z}; \bftheta_h)$ is convex and Lipschitz-smooth in norm $\| \cdot \|$ wrt $\bfz$ provided: i) Element-wise activation $\sigma: \real \mapsto \real$ is convex, Lipschitz-continuous, Lipschitz-smooth, and non-decreasing; ii) $\{ \bfW_i \}_{i=1}^I$ are bounded wrt the dual norm $\| \cdot \|_*$ and element-wise non-negative; and iii) $\{ \bfM_i \}_{i=1}^I$ are bounded wrt $\| \cdot \|_*$. 
\end{theorem}

The term $\bfM_i^\top \bfz$ in~\eqref{eq:NN-h*-forward} implements a skip connection from the input to $i$-th layer, which improves NN expressiveness. The following remark demonstrates how to construct such an NN via proper parameterization. 

\begin{remark}[Valid NN parametrization]
\label{remark:valid-NN}
The three conditions under Theorem~\ref{thm:cvx-Lip-NN} are mild, and they can be readily satisfied. For instance, choices of $\sigma$ obeying i) include ELU~\cite{ELU} and Softplus $\sigma(\cdot)=\log(1+\exp(\cdot))$. In addition, conditions ii) and iii) can be attained via parameterization $\bfW_i := \sigma_W(\check{\bfW}_i)$ and $\bfM_i  := \sigma_M(\check{\bfM}_i)$, where $\sigma_W$ is non-negative and bounded (e.g., sigmoid), $\sigma_M$ is bounded (e.g., $\tanh$), and parameter of $h^*$ is defined as $\bftheta_h := \{ \check{\bfW}_i, \check{\bfM}_i, \bfb_i \}_{i=1}^I$. 
\end{remark}

Using these choices, the next corollary asserts that the resultant NN has a universal Lipschitz-smoothness constant wrt $\bfz$ regardless of $\bftheta_h$. 

\begin{corollary}[Universal Lipschitz-smoothness]
\label{cor:univ-Lip}
Suppose conditions in Theorem~\ref{thm:cvx-Lip-NN} hold, with weight bounds in conditions ii) and iii) fixed. For any $\bftheta_h$, there exists a universal Lipschitz-smoothness constant for $h^*$. 
\end{corollary}

In fact, $h^*$ in Theorem~\ref{thm:cvx-Lip-NN} can be proved also Lipschitz-continuous wrt $\bfz$; see Appendix~\ref{app:proof-prop-corollary}. To enable modeling Lipschitz-smooth yet not Lipschitz-continuous functions, one can add an additional quadratic term to the output of the NN; that is, $h^*(\bfz;\bftheta_h) := \bfa_I (\bfz; \bftheta_h) + \frac{1}{2} \bfz^\top \bfP (\bftheta_P) \bfz$, where $\bfP \succeq 0$ and $\| \bfP \|$ is bounded, and $\bftheta_h$ is augmented with $\bftheta_P$. This not only satisfies the desired convexity and Lipschitz-smoothness requirements, but also enhances the expressiveness of $h^*$ by explicitly incorporating the quadratic loss curvature model used in PGD. 

\subsection{Algorithm design with guaranteed convergence}
Having modeled $h^* (\cdot; \bftheta_h)$ with an NN, here we develop the proposed approach dubbed meta-learning with MiDA (MetaMiDA). Utilizing the dual update~\eqref{eq:MDA-dual} and $h^*$ provided by Theorem~\ref{thm:cvx-Lip-NN}, our optimization objective is
\begin{subequations}
\label{eq:MetaMDA-obj}
\begin{align}
\label{eq:MetaMDA-obj-meta}
	\min_{\bftheta_z, \bftheta_h} \;&\expect_t\metaloss_t (\bftheta_z, \bftheta_h) := \expect_t \loss{val}_t(\nabla_1 h^*(\bfz_t^K (\bftheta_z, \bftheta_h); \bftheta_h)) \\
\label{eq:MetaMDA-obj-task}
	\text{s.t.}\;\; &\bfz_t^0 (\bftheta_z) = \bftheta_z, \;  t=1,\ldots,T, \\
	& \bfz_t^{k+1} (\bftheta_z, \bftheta_h) = \bfz_t^k (\bftheta_z, \bftheta_h) - \nonumber \\
	&\quad \alpha \nabla \loss{trn}_t(\nabla_1 h^* (\bfz_t^k (\bftheta_z, \bftheta_h); \bftheta_h)), \, k=0,\ldots,K-1. \nonumber
\end{align}
\end{subequations}
This can be solved using the standard alternating optimizer, where~\eqref{eq:MetaMDA-obj-meta} is optimized via mini-batch stochastic GD (SGD). The pseudo-codes of the resultant MetaMiDA approach are summarized in Algorithm~\ref{alg:MetaMiDA}.

\begin{algorithm}[t]
\caption{MetaMiDA}
\label{alg:MetaMiDA}
\begin{algorithmic}[1]
\REQUIRE datasets $\{ \data{}_t \}_{t=1}^T$, maximum iterations $K$ and $R$, step sizes $\alpha$ and $\{ (\beta_1^r, \beta_2^r) \}_{r=1}^R$, and batch size $B$. 
\STATE Randomly initialize $\bftheta^0 := \{ \bftheta_z^0, \bftheta_h^0 \}$.
\FOR{$r=0,\ldots,R-1$}
	\STATE Randomly sample $\batch^r \subset \{ 1,2,\ldots,T \}$ with $|\batch^r| = B$. 
	\FOR{$t \in \batch^r$}
		\STATE Initialize $\bfz_t^0 (\bftheta^r) = \bftheta_z^r$. 
		\FOR{$k=0,\ldots,K-1$}
			\STATE Map $\bfphi_t^k (\bftheta^r) = \nabla_1 h^*(\bfz_t^k (\bftheta^r); \bftheta_h^r)$.
			\STATE Descend $\bfz_t^{k+1} (\bftheta^r) = \bfz_t^k (\bftheta^r) - \alpha \nabla \loss{trn}_t(\bfphi_t^k (\bftheta^r))$.
		\ENDFOR
		\STATE Map $\bfphi_t^K (\bftheta^r) = \nabla_1 h^*(\bfz_t^K (\bftheta^r); \bftheta_h^r)$. 
	\ENDFOR
	\STATE Update $\bftheta_z^{r+1} = \bftheta_z^r - \frac{\beta_1^r}{B} \sum_{t \in \batch^r} \nabla_1 \metaloss_t (\bftheta^r)$, \\ \hspace{.45cm} and $\bftheta_h^{r+1} = \bftheta_h^r - \frac{\beta_2^r}{B} \sum_{t \in \batch^r} \nabla_2 \metaloss_t (\bftheta^r)$. 
\ENDFOR
\ENSURE $\bftheta^R = \{ \bftheta_z^R, \bftheta_h^R \}$
\end{algorithmic}
\end{algorithm}

Next, we analyze the convergence of Algorithm~\ref{alg:MetaMiDA}. While existing convergence results in meta-learning focus exclusively on GD-based approaches~\cite{converge-onestep-MAML,converge-multistep-MAML,converge-partial-param}, our approach here pertains to the broader family of MiDA-based methods, which also includes all the PGD-based ones. Comparing to~\eqref{eq:MAML-obj}, our objective~\eqref{eq:MetaMDA-obj} contains an extra parameter $\bftheta_h$, which thus complicates the analysis. As the optimization of~\eqref{eq:MetaMDA-obj-meta} is based on SGD, the analysis will apply Theorem~\ref{thm:cvx-Lip-NN} with $\| \cdot \|_2$ norm. For notational simplicity, let $\loss{set} := \expect_t [\loss{set}_t]$ denote the expected NLL with $\mathrm{set} \in \{ \mathrm{trn}, \mathrm{val} \}$, and meta-loss $\metaloss := \expect_t \metaloss_t$. Define $\bftheta := \{ \bftheta_z, \bftheta_h \}$, $h^*(\bftheta) := h^*(\bftheta_z; \bftheta_h)$, $\nabla_1 h^* (\bftheta) := \nabla_1 h^*(\bftheta_z; \bftheta_h)$, and $\nabla_2 h^* (\bftheta) := \nabla_2 h^*(\bftheta_z; \bftheta_h)$. Likewise, define $\metaloss (\bftheta) := \metaloss(\bftheta_z, \bftheta_h)$, and $\nabla_1 \metaloss (\bftheta), \nabla_2 \metaloss (\bftheta)$ respectively the partial derivatives wrt $\bftheta_z, \bftheta_h$. Our analysis is built on top of the following three technical assumptions. 

\begin{assumption}[Loss functions]
\label{as:loss-fn}
	For each $t = 1,\ldots,T$ and $\forall \bfphi, \tilde{\bfphi} \in \real^d$, loss function $\loss{set}_t, \, \mathrm{set} \in \{ \mathrm{trn}, \mathrm{val} \}$ has \\
	i) finite lower bound $\loss{set}_t (\bfphi) > -\infty$; \\
	ii) Lipschitz-continuous gradient $\| \nabla \loss{set}_t (\bfphi) - \nabla \loss{set}_t (\tilde{\bfphi}) \|_2 \allowbreak \le G_\loss{} \| \bfphi - \tilde{\bfphi} \|_2$; \\
	iii) Lipschitz-continuous Hessian $\| \nabla^2 \loss{set}_t (\bfphi) - \nabla^2 \loss{set}_t (\tilde{\bfphi}) \|_2 \le H_\loss{} \| \bfphi - \tilde{\bfphi} \|_2$; and\\
	iv) Lipschitz-smooth composition $\| \nabla (\loss{val}_t \circ \nabla_1 h^*) (\bftheta) - \nabla (\loss{val}_t \circ \nabla_1 h^*) (\tilde{\bftheta}) \|_2 \le G_{\loss{}h} \| \bftheta - \tilde{\bftheta} \|_2, \;\forall \bftheta, \tilde{\bftheta} \in \real^D$. 
\end{assumption}

\begin{assumption}[Bounded variance]
\label{as:stat-property}
For each $t=1,\ldots,T$ and $\forall \bfphi \in \real^d$, it holds that $\expect_t \| \nabla \loss{val}_t (\bfphi) - \nabla \loss{val} (\bfphi) \|_2^2 \le \sigma^2$. 
\end{assumption}

\begin{assumption}[Inverse mirror map]
\label{as:inv-mirror-map}
For $\forall \bftheta, \tilde{\bftheta} \in \real^D$, $h^*$ has Lipschitz-continuous \\
i) partial gradient $\| \nabla_1 h^* (\bftheta) - \nabla_1 h^* (\tilde{\bftheta}) \|_2 \le G_{h} \| \bftheta - \tilde{\bftheta} \|_2$; and \\
ii) mixed Jacobian $\| \nabla \nabla_1 h^* (\bftheta) - \nabla \nabla_1 h^* (\tilde{\bftheta}) \|_2 \le H_{h} \| \bftheta - \tilde{\bftheta} \|_2$.
\end{assumption}

\begin{remark}[Mild assumptions]
Assumptions~\ref{as:loss-fn}-\ref{as:inv-mirror-map} are mild and common in not only meta-learning~\cite{converge-onestep-MAML,converge-multistep-MAML,converge-partial-param} but also generic bilevel optimization~\cite{bilevel-converge-design,bilevel-tighter,TTSA}. Assumptions~\ref{as:loss-fn} and~\ref{as:inv-mirror-map} merely assume Lipschitz-continuity of the gradients and Hessians of $\loss{set}_t$ and $h^*$, without any premise on their own Lipschitz-continuity or the convexity of $\loss{set}_t$. It is also noteworthy that the Lipschitz-continuity of $\nabla_1 h^* (\bftheta)$ wrt $\bftheta_z$ has already been proved in Corollary~\ref{cor:univ-Lip}. As $\| \nabla_1 h^* (\bftheta) - \nabla_1 h^* (\tilde{\bftheta}) \|_2 \le \| \nabla_1 h^* (\bftheta) - \nabla_1 h^* (\tilde{\bftheta}_z; \bftheta_h) \|_2 + \| \nabla_1 h^* (\tilde{\bftheta}_z; \bftheta_h) - \nabla_1 h^* (\tilde{\bftheta}) \|_2$, a sufficient condition for Assumption~\ref{as:inv-mirror-map}.i) to hold is the Lipschitz-continuity of $\nabla_1 h^* (\bftheta)$ wrt $\bftheta_h$, which is also mild. Assumption~\ref{as:stat-property} is standard in stochastic optimization for analyzing convergence of SGD. 
\end{remark}

The key challenge in the convergence analysis is the unbounded smoothness of $\metaloss$. This property is characterized in the next proposition. 

\begin{proposition}[Unbounded smoothness of meta-loss]
\label{prop:meta-smooth}
Suppose Assumptions~\ref{as:loss-fn} and~\ref{as:inv-mirror-map} hold. Define constants $\gamma := 1 + \alpha G_h G_\loss{}$, $C_{G,1} := \gamma^K (\gamma^K - 1) (G_h\frac{H_\loss{}}{G_\loss{}} + \frac{H_h}{G_h})$, $C_{G, 2} := \gamma^{K-1} [(\gamma - 1) (G_h \frac{H_\loss{}}{G_\loss{}} + \frac{H_h}{G_h}) K + (\gamma^K - 1 - \alpha K G_h G_\loss{}) (G_h\frac{H_\loss{}}{G_\loss{}} + \frac{H_h}{G_h})]$, and vector $\bfg_t^K := \nabla_1 (\loss{val}_t \circ \nabla_1 h^*) (\bfz_t^K; \bftheta_h)$. It holds that
\begin{subequations}
\label{eq:meta-smooth}
\begin{align}
\label{eq:meta-smooth-z}
	\| \nabla_1 \metaloss (\bftheta) - \nabla_1 \metaloss (\tilde{\bftheta}) \|_2 &\le G_{\metaloss,1} \| \bftheta - \tilde{\bftheta} \|_2, \\
\label{eq:meta-smooth-h}
	\| \nabla_2 \metaloss (\bftheta) - \nabla_2 \metaloss (\tilde{\bftheta}) \|_2 &\le G_{\metaloss,2} \| \bftheta - \tilde{\bftheta} \|_2
\end{align}
\end{subequations}
where $G_{\metaloss,j} := C_{G,j} \expect_t \| \bfg_t^K \|_2 + \gamma^{2K} G_{\loss{}h}, \; j = 1,2$. 
\end{proposition}

Proposition~\ref{prop:meta-smooth} points out that the smoothness constant $G_{\metaloss,j}$ grows linearly with $\expect_t \| \bfg_t^K \|_2$, and thus it can be unbounded. The standard remedy in SGD is to estimate $G_{\metaloss,j}$, and adjust $(\beta_1^r, \beta_2^r)$ on the fly. This estimation relies on a mini-batch $\hat{\batch}^r \subset \{ 1,2,\ldots,T \}$ of cardinality $\hat{B}$ that is sampled independently of $\batch^r$, to yield the estimator 
\begin{equation}
\label{eq:def-Ghat}
\hat{G}_{\metaloss,j}^r := \frac{C_{G,j}}{\hat{B}} \sum_{t \in \hat{\batch}^r} \| \bfg_t^K \|_2 + \gamma^{2K} G_{\loss{}h},\; j = 1,2. 
\end{equation}
It can be inferred from the definition~\eqref{eq:def-Ghat} that $\expect_{\hat{\batch}^r} \hat{G}_{\metaloss,j}^r = G_{\metaloss,j}$ is unbiased. Next, we analyze the complexity of this estimator, and compare it with existing approaches. 

\begin{remark}[Complexity of smoothness estimator]
Estimator~\eqref{eq:def-Ghat} requires sampling an additional batch of tasks. To compute $\bfg_t^K$ for each $t \in \hat{\batch}^r$, one needs to acquire $\bfz_t^K$ via $K$ mirror descent steps. Fortunately, this will not increase time complexity since the mirror descent~\eqref{eq:MetaMDA-obj-task} is computed \emph{in parallel} across tasks in $\batch^r \cup \hat{\batch}^r$. Moreover, it is seen from lines 4-12 of Algorithm~\ref{alg:MetaMiDA} that the computation of meta-gradient $\nabla \metaloss = [\nabla_1 \metaloss^\top, \nabla_2 \metaloss^\top]^\top$ relies solely on $\batch^r$. Hence, $\bfz_t^k$ can be updated \emph{in place} for $t \in \hat{\batch}^r$, without creating computation graphs and storing second-order information for backpropagation. Additionally, it will be shown soon in Corollary~\ref{cor:converge-rate} that $\hat{B} \ll B$. As a result, this estimator only contributes to a marginal increase ($< 2\%$ in our numerical tests) of the overall space complexity. 
\end{remark}

\begin{remark}[Comparison with existing methods]
\label{remark:comp-exist}
The smoothness estimator~\eqref{eq:def-Ghat} distinguishes this paper from current works by relying on \emph{fewer assumptions} and providing \emph{tighter bounds}. 
Existing smoothness estimators~\cite{converge-onestep-MAML,converge-multistep-MAML} are linear functions of $\| \nabla \loss{val}_t (\bfphi_t^0) \|_2$ instead of $\| \bfg_t^K \|_2$. Although these estimators can be computed at the initialization $\bfphi_t^0$, they rely on the extra assumption that $\data{trn}_t$ is similar enough to $\data{val}_t$ so $\| \nabla \loss{trn}_t (\bfphi) - \nabla \loss{val}_t (\bfphi) \|_2$ is bounded for $\forall t$ and $\forall \bfphi \in \real^d$. 
To confirm the tightness of the proposed estimator, consider $\alpha = 0$ or $K = 0$ so that $\metaloss = \loss{val} \circ h$ has constant smoothness $G_{\loss{}h}$. In this case, taking $\alpha = 0$ or $K = 0$ in~\eqref{eq:def-Ghat} yields coefficient $C_{G,j} = 0$, and thus our estimator $\hat{G}_{\metaloss, j} = G_{\loss{}h}$ is tight. With the same setup, the coefficient of $\| \nabla \loss{val}_t (\bfz_t^0) \|_2$ are yet greater than $0$; see e.g.,~\cite[Proposition 7]{converge-multistep-MAML}, hence resulting in a loose and even unbounded smoothness estimator. 
Similar to~\cite{converge-onestep-MAML,converge-multistep-MAML}, calculation of~\eqref{eq:def-Ghat} requires Lipschitz constants in Assumptions~\ref{as:loss-fn} and~\ref{as:inv-mirror-map}, which are hard to estimate especially for large NNs. In practice, these two scalars can be viewed as hyperparameters and are obtained through grid search. As an alternative, if $\loss{trn}_t$ is further assumed Lipschitz-continuous, $\hat{G}_{\metaloss,j}^r$ will boil down to a constant, thus enabling a constant learning rate~\cite{converge-partial-param}. 
\end{remark}

An ideal choice is $\beta_j^r = 1 / G_{\metaloss,j}^r$, which is unfortunately infeasible to compute. 
Intuitively, if $\hat{B}$ is moderately large,~\eqref{eq:def-Ghat} could provide an accurate estimation of $G_{\metaloss,j}^r$, and one can thus set $\beta_j^r \propto 1 / \hat{G}_{\metaloss,j}^r$. The following proposition bounds the first- and second-order moments of this choice. 

\begin{proposition}[Moments of learning rate estimator]
\label{prop:Ghat-moment}
Suppose Assumptions~\ref{as:loss-fn}-\ref{as:inv-mirror-map} hold, and define $\zeta := 2\alpha + \frac{(\gamma^K - \gamma)(\sqrt{T} + 1)}{G_h G_{\loss{}}}$. If $\hat{B} \ge \max_{j=1,2} \frac{2 C_{G,j}^2 [ G_{\loss{}h} \zeta + G_h (1+\sqrt{T})]^2}{\gamma^{4K} G_{\loss{}h}^2 \sqrt{T}} \sigma^2$, then
\begin{equation*}
	\expect_{\hat{\batch}^r} \bigg[ \frac{1}{\hat{G}_{\metaloss,j}^r} \bigg] \ge \frac{1}{G_{\metaloss,j}},~ \expect_{\hat{\batch}^r} \bigg[ \frac{1}{(\hat{G}_{\metaloss,j}^r)^2} \bigg] \le \frac{\sqrt{T}+1}{G_{\metaloss,j}^2}.
\end{equation*}
\end{proposition}

It is worth noting that while $T$ can be large (e.g., $T = 8,000$~\cite{MetaOptNet}), its square root is relatively small ($\sqrt{8000} \approx 89$). Next, we account for the stochasticity of SGD by bounding the second-order moment of the meta-gradient norm. 

\begin{proposition}[Second-order moment of meta-gradient]
\label{prop:metagrad-2nd-moment}
Suppose Assumptions~\ref{as:loss-fn}-\ref{as:inv-mirror-map} hold, and let $C_{\metaloss,1} := \frac{\gamma^K}{2 - \gamma^K}$, $C_{\metaloss,2} := G_{\loss{}h} \zeta + G_h \sqrt{T}$. If $\alpha < \frac{2^{1/K} - 1}{G_h G_\loss{}}$, then it holds for any given $\bftheta \in \real^D$ that
\begin{subequations}
\begin{align}
\label{eq:metagrad-2nd-moment-z}
\expect_t^\frac{1}{2} \| \nabla_1 \metaloss_t (\bftheta) \|_2 
&\le C_{\metaloss,1} \| \nabla_1 \metaloss (\bftheta) \|_2 + C_{\metaloss,1} C_{\metaloss,2} \sigma, \\
\label{eq:metagrad-2nd-moment-h}
\expect_t^\frac{1}{2} \| \nabla_2 \metaloss_t (\bftheta) \|_2 
&\le (C_{\metaloss,1} - 1) \| \nabla_1 \metaloss (\bftheta) \|_2 + \nonumber \\
	&\qquad C_{\metaloss,1} C_{\metaloss,2} \sigma + \| \nabla_2 \metaloss (\bftheta) \|_2. 
\end{align}
\end{subequations}
\end{proposition}

Building upon Propositions~\ref{prop:meta-smooth}-\ref{prop:metagrad-2nd-moment}, we now establish the convergence of our MetaMiDA approach. 

\begin{theorem}[Convergence of MetaMiDA]
\label{thm:converge}
Consider Algorithm~\ref{alg:MetaMiDA} with meta-learning rate $\beta_j^r = \frac{1}{C_\beta \hat{G}_{\metaloss,j}^r},\;j=1,2$, where $C_\beta > \frac{\sqrt{T}+1}{2}$. Suppose that Assumptions~\ref{as:loss-fn}-\ref{as:inv-mirror-map}, and the conditions on $\hat{B}$ and $\alpha$ in Propositions~\ref{prop:Ghat-moment} and~\ref{prop:metagrad-2nd-moment} hold. Define constants $C_{B,1} := \frac{2C_\beta}{\sqrt{T}+1}$ and $C_{B,2} := 2 C_{\metaloss,1}^2 -1 + 3 (C_{\metaloss,1} - 1)^2 \max\{ \frac{C_{G,1}}{C_{G,2}}, 1 \}$. If $B \ge \frac{\max \{ C_{B,2}, 2 \}}{C_{B,1} - 1}$, it holds that
\begin{align*}
\expect \| \nabla_1 \metaloss (\bftheta^\rho) \|_2 
&\le \frac{1}{2\eta_1} \Big( \frac{\Delta}{R} + \frac{\eta_5}{B} \Big) +  \\
&\qquad \sqrt{ \frac{1}{4\eta_1^2} \Big(\frac{\Delta}{R} + \frac{\eta_5}{B} \Big)^2 + \frac{\eta_2}{\eta_1} \Big( \frac{\Delta}{R} + \frac{\eta_5}{B} \Big)}, \nonumber \\
\expect \| \nabla_2 \metaloss (\bftheta^\rho) \|_2 
&\le \sqrt{\frac{1}{\eta_3} \big( \expect \| \nabla_1 \metaloss (\bftheta^\rho) \|_2 + \eta_4 \big) \Big( \frac{\Delta}{R} + \frac{\eta_5}{B} \Big)}.
\end{align*}
where $\rho \in \{0,\ldots,R-1\}$ is a discrete uniform random variable, $\Delta := \expect [\metaloss(\bftheta^0) - \inf_{\bftheta} \metaloss(\bftheta)]$, and $\{ \eta_i \}_{i=1}^5$ are defined in~\eqref{eq:app-def-eta}. 
\end{theorem}

By specifying the values of $\alpha$ and $C_\beta$, the following corollary provides a simplified version of Theorem~\ref{thm:converge}. 

\begin{corollary}[Convergence rate]
\label{cor:converge-rate}
Consider Theorem~\ref{thm:converge} with  $\alpha = \frac{(1 + T^{-1/2})^{1/K} - 1}{G_h G_{\loss{}}}$ and $C_\beta = \sqrt{T}$. If batch sizes satisfy $B = \Omega (\sigma^2 \epsilon^{-2})$ and $\hat{B} = \Omega (\sigma^2)$, it takes $R = \mathcal{O} (\epsilon^{-2})$ iterations to reach an $\epsilon$-stationary point; that is, $\expect \| \nabla \metaloss (\bftheta^\rho) \|_2 \le \epsilon$.
\end{corollary}

Although MetaMiDA relies on more complicated task-level updates~\eqref{eq:MetaMDA-obj-task} and fewer assumptions, our analysis still guarantees the same convergence rate as for GD-based meta-learning ~\cite{converge-onestep-MAML, converge-multistep-MAML, converge-partial-param}. 



\begin{table*}[t]
\caption{Performance comparison of MetaMiDA against PGD-based meta-learning approaches. The highest accuracy for each column is marked with dark gray, and mean accuracies within its $95\%$ confidence intervals are marked with light gray. GAP~\cite{GAP} that relies on SVD has been excluded from comparison due to its significantly increased complexity.}
 \vskip -0.2in
\label{tab:comp-PGD}
\begin{center}
\begin{tabular}{lcccccc}
\toprule
\multirow{2}*{Method} & Task-level & \multirow{2}*{Loss geometry prior} & \multicolumn{2}{c}{MiniImageNet, $5$-class} & \multicolumn{2}{c}{TieredImageNet, $5$-class} \\
 & optimizer & & $1$-shot ($\%$) & $5$-shot ($\%$) & $1$-shot ($\%$) & $5$-shot ($\%$) \\
\midrule
MAML~\cite{MAML} & GD & Quadratic \& isotropic & $48.70_{\pm 1.84}$ & $63.11_{\pm 0.92}$ & $51.67_{\pm 1.81}$ & $70.30_{\pm 1.75}$ \\
\hline
MetaSGD~\cite{MetaSGD} & PGD & Quadratic \& diagonal & $50.47_{\pm 1.87}$ & $64.03_{\pm 0.94}$ & $50.92_{\pm 0.93}$ & $69.28_{\pm 0.80}$ \\
MC~\cite{MetaCurvature} & PGD & Quadratic \& block-diagonal & $54.08_{\pm 0.88}$ & $67.99_{\pm 0.73}$ & N/A & N/A \\
WarpGrad~\cite{WarpGrad} & PGD & Quadratic \& full & $52.3_{\pm 1.6}$ & $68.4_{\pm 1.2}$ & $57.2_{\pm 1.8}$ & $74.1_{\pm 1.4}$ \\
ModGrad~\cite{ModGrad} & PGD & Quadratic \& diagonal & $53.20_{\pm 0.86}$ & $69.17_{\pm 0.69}$ & N/A & N/A \\
PAMELA~\cite{PAMELA} & PGD & Quadratic \& diagonal & $53.50_{\pm 0.89}$ & $70.51_{\pm 0.67}$ & $54.81_{\pm 0.88}$ & \lightcell $74.39_{\pm 0.71}$ \\
Sparse-MAML+~\cite{sparse-MAML} & PGD & Quadratic \& sparse diagonal & $51.04_{\pm 1.16}$ & $68.05_{\pm 1.65}$ & N/A & N/A \\
Approximate GAP~\cite{GAP} & PGD & Quadratic \& block-diagonal & $53.52_{\pm 0.88}$ & $70.75_{\pm 0.67}$ & $56.86_{\pm 0.91}$ & \lightcell $74.41_{\pm 0.72}$ \\
\hline
GAP ($1.42\times$ slower)~\cite{GAP} & PGD w/ SVD & Riemannian metric & $54.86_{\pm 0.85}$ & $71.55_{\pm 0.61}$ & $57.60_{\pm 0.93}$ & $74.90_{\pm 0.68}$ \\
\hline
MetaMiDA (ours) & MiDA & Convex \& Lipschitz-smooth & \deepcell $56.04_{\pm 1.41}$ & \deepcell $72.06_{\pm 0.67}$ & \deepcell $58.52_{\pm 1.45}$ & \deepcell $75.07_{\pm 0.70}$ \\
\bottomrule
\end{tabular}
\end{center}
\end{table*}

\begin{table*}[t]
\vskip -0.15in
\caption{Performance comparison of MetaMiDA against meta-learning approaches using a WRN-28-10. ``Center'' stands for features from the central crop, while ``multiview'' means features averaged over four corners, central crops, and horizontally mirrored images. $^\dagger$ indicates that both meta-training and meta-validation tasks are used in the meta-training phase.}
\vskip -0.2in
\label{tab:LEO-embed}
\begin{center}
\begin{tabular}{lcccccc}
\toprule
\multirow{2}*{Method} & \multicolumn{2}{c}{MiniImageNet, center, $5$-class} & \multicolumn{2}{c}{MiniImageNet, multiview, $5$-class} & \multicolumn{2}{c}{TieredImageNet, center, $5$-class} \\
 & $1$-shot ($\%$) & $5$-shot ($\%$) & $1$-shot ($\%$) & $5$-shot ($\%$) & $1$-shot ($\%$) & $5$-shot ($\%$) \\
\midrule
MetaSGD~\cite{MetaSGD} & $56.58_{\pm 0.21}$ & $68.84_{\pm 0.19}$ & N/A & N/A & $59.75_{\pm 0.25}$ & $69.04_{\pm 0.22}$ \\
LEO$^\dagger$~\cite{WarpGrad} & $61.76_{\pm 0.08}$ & $77.58_{\pm 0.12}$ & $63.97_{\pm 0.20}$ & $79.49_{\pm 0.70}$ & \lightcell $66.33_{\pm 0.05}$ & $81.44_{\pm 0.09}$ \\
MC~\cite{MetaCurvature} & $61.22_{\pm 0.10}$ & $75.92_{\pm 0.17}$ & N/A & N/A & $66.20_{\pm 0.10}$ & $82.21_{\pm 0.08}$ \\
MC$^\dagger$~\cite{MetaCurvature} & $61.85_{\pm 0.10}$ & $77.02_{\pm 0.11}$ & $64.40_{\pm 0.10}$ & $80.21_{\pm 0.10}$ & \lightcell $67.21_{\pm 0.10}$ & $82.61_{\pm 0.08}$ \\
\hline
MetaMiDA (ours) & \deepcell $63.64_{\pm 1.35}$ & \deepcell $78.58_{\pm 0.58}$ & \deepcell $66.40_{\pm 1.32}$ & \deepcell $84.95_{\pm 0.51}$ & \deepcell $67.56_{\pm 1.28}$ & \deepcell $84.04_{\pm 0.55}$ \\
\bottomrule
\end{tabular}
\end{center}
\vskip -0.1in
\end{table*}

\section{Numerical tests}
In this section, we evaluate experimentally the performance of MetaMiDA, and gain insights into the underlying reason for its effectiveness. All the codes are performed on a desktop with NVIDIA RTX A5000 GPUs, and a server group with NVIDIA A100 GPUs. Implementation details including hyperparameters and $h^*$ are deferred to Appendix~\ref{app:hyperparams}. 

\subsection{Benchmark datasets}
Four popular few-shot classification datasets are considered for performance assessment. The term ``shot'' signifies the number of per-class training data for each $t$. 

\textbf{MiniImageNet}~\cite{MatchingNets} is a subset of the full ImageNet (ILSVRC-12) dataset~\cite{ImageNet} consisting of $60,000$ labeled images. These images are sampled from $100$ classes, each containing $600$ instances. The dataset is split into $64$, $16$, and $20$ disjoint classes according to~\cite{Meta-LSTM}, which can be respectively accessed during the training, validation, and testing phases of meta-learning (aka, meta-training/-validation/-testing) to form few-shot classification tasks. Following the standard preprocessing setup~\cite{Meta-LSTM,MAML}, all images are cropped and resized to $84 \times 84$ pixels. 

\textbf{TieredImageNet}~\cite{tieredImageNet} is a larger subset of the full ImageNet~\cite{ImageNet}, which is composed of $779,165$ labeled images sampled from $608$ distinct classes. These classes are partitioned into 34 categories according to the hierarchy of ImageNet dataset, each category containing $10$ to $30$ classes. The categories are further grouped into $3$ disjoint subsets: $20$ for meta-training, $6$ for meta-validation, and $8$ for meta-testing~\cite{tieredImageNet}. Similar to miniImageNet, the preprocessing also alters the images to $84 \times 84$ pixels. 

\textbf{Caltech-UCSD Birds-200-2011 (CUB)}~\cite{CUB} dataset comprises $11,788$ images of $200$ bird species. Different from ImageNet, which focuses on nature objects of various shapes and colors, the data in CUB are fine-grained images with rich textures and patterns. Following the dataset split from~\cite{baseline++}, $100$ species are used for meta-training, $50$ for meta-validation, and $50$ for meta-testing. Likewise, all the images are preprocessed to the standard $84 \times 84$ size. 

\textbf{Cars}~\cite{Cars} is another fine-grained dataset focusing exclusively on cars of various brands and models. Similar to the CUB dataset, the model is expected to capture and distinguish the subtle difference in appearance of cars. The dataset is formed by $16,185$ images belonging to $196$ classes. The data is partitioned into $8,144$ training images and $8,041$ testing ones, where each class has been split roughly half-and-half. The class split is from~\cite{cross-domain}, with $98$, $49$, and $49$ classes allocated for meta-training, meta-validation, and meta-testing, respectively.

\subsection{Comparison with PGD-based meta-learning}
To showcase the benefit of the advocated generic loss geometries, the first test compares MetaMiDA with existing PGD-based meta-learning approaches. The task-specific model is a standard $4$-block convolutional NN (CNN)~\cite{MatchingNets, MAML}. Each block comprises a $3 \times 3$ convolutional layer, a batch normalization layer, a ReLU activation, and a $2 \times 2$ max pooling layer. After these four blocks, a linear regressor with softmax activation is appended to perform classification. For a fair comparison, MetaMiDA is implemented with a constant meta-learning rate $(\beta_1, \beta_2)$ instead of its per-step estimation via an additional $\hat{\batch}^r$. This choice essentially assumes $\loss{trn}_t$ is Lipschitz-continuous, as discussed in Remark~\ref{remark:comp-exist}. Additionally, we set $K = 5$, which is consistent with all competing alternatives. 

The test is conducted on the 5-class 1-shot and 5-shot classification tasks randomly generated from miniImageNet and tieredImageNet. The metric is the accuracy averaged on $1,000$ random test tasks with its $95\%$ confidence interval. It is observed from Table~\ref{tab:comp-PGD} that our MetaMiDA not only consistently outperforms all the PGD-based methods, but also slightly surpasses the GAP algorithm~\cite{GAP} that requires burdensome SVD computations with a $40\%$ increased time complexity. This demonstrates the effectiveness of learning generalized geometries beyond quadratic ones. 

Similar to~\cite{MetaSGD,MetaCurvature}, MetaMiDA can be also leveraged to fine-tune pre-trained large-scale models. The next test evaluates MetaMiDA on miniImageNet and tieredImageNet using a pre-trained Wide ResNet (WRN)-28-10 model~\cite{LEO}. Following the setups in~\cite{LEO,MetaCurvature}, the task-level optimization~\eqref{eq:MetaMDA-obj-task} fine-tunes only the last block of the WRN, while the feature extractor is frozen to accelerate the meta-training procedure. The results are summarized in Table~\ref{tab:LEO-embed}, where MetaMiDA again presents a remarkable performance gain over alternatives. This test confirms that a generic and expressive loss geometry model can enhance the empirical performance of meta-learning.

\begin{table*}[t]
\vskip -0.1in
\caption{Results of cross-domain adaptation. Models are meta-trained on miniImageNet, and then meta-tested on tieredImageNet, CUB, and Cars. The highest accuracy and mean accuracies within its $95\%$ confidence intervals are marked with dark and light gray, respectively. GAP~\cite{GAP} that relies on SVD has been excluded from comparison due to its significantly increased complexity.}
\vskip -0.2in
\label{tab:cross-domain}
\begin{center}
\begin{tabular}{lccccccc}
\toprule
\multirow{2}*{Method} & Task-level & \multicolumn{2}{c}{TieredImageNet, $5$-class} & \multicolumn{2}{c}{CUB, $5$-class} & \multicolumn{2}{c}{Cars, $5$-class} \\
 & optimizer & $1$-shot ($\%$) & $5$-shot ($\%$) & $1$-shot ($\%$) & $5$-shot ($\%$) & $1$-shot ($\%$) & $5$-shot ($\%$) \\
\midrule
MAML~\cite{MAML} & GD & $51.61_{\pm 0.20}$ & $65.76_{\pm 0.27}$ & $40.51_{\pm 0.08}$ & $53.09_{\pm 0.16}$ & $33.57_{\pm 0.14}$ & $44.56_{\pm 0.21}$ \\
ANIL~\cite{ANIL} & GD & $52.82_{\pm 0.29}$ & $66.52_{\pm 0.28}$ & $41.12_{\pm 0.15}$ & $55.82_{\pm 0.21}$ & $34.77_{\pm 0.31}$ & $46.55_{\pm 0.29}$ \\
BOIL~\cite{BOIL} & GD & $53.23_{\pm 0.41}$ & $69.37_{\pm 0.23}$ & $44.20_{\pm 0.15}$ & $60.92_{\pm 0.11}$ & $36.12_{\pm 0.29}$ & $50.64_{\pm 0.22}$ \\
\hline
Sparse-MAML+~\cite{sparse-MAML} & PGD & $53.91_{\pm 0.67}$ & $69.92_{\pm 0.21}$ & $43.43_{\pm 1.04}$ & $62.02_{\pm 0.78}$ & \lightcell $37.14_{\pm 0.77}$ & $53.18_{\pm 0.44}$ \\
Approximate GAP~\cite{GAP} & PGD & \lightcell $57.47_{\pm 0.99}$ & $71.66_{\pm 0.76}$ & $43.77_{\pm 0.79}$ & $62.92_{\pm 0.73}$ & \lightcell $37.00_{\pm 0.75}$ & $53.28_{\pm 0.76}$ \\
\hline
GAP ($1.42\times$ slower)~\cite{GAP} & PGD w/ SVD & $58.56_{\pm 0.93}$ & $72.82_{\pm 0.77}$ & $44.74_{\pm 0.75}$ & $64.88_{\pm 0.72}$ & $38.44_{\pm 0.77}$ & $55.04_{\pm 0.77}$ \\
\hline
MetaMiDA (ours) & MiDA & \deepcell $58.42_{\pm 1.42}$ & \deepcell $72.50_{\pm 0.74}$ & \deepcell $45.54_{\pm 1.45}$ & \deepcell $64.66_{\pm 0.73}$ & \deepcell $38.36_{\pm 1.41}$ & \deepcell $54.19_{\pm 0.75}$ \\
\bottomrule
\end{tabular}
\end{center}
\vskip -0.1in
\end{table*}

\subsection{Cross-domain generalization}
We next assess MetaMiDA's performance in a more challenging and practical scenario, dubbed cross-domain few-shot learning, in which meta-training and meta-testing tasks are from two different yet related domains. By deliberately increasing the domain gap, this test measures the overfitting of the learned prior to a specific domain (aka, meta-overfitting). In practice, such a behavior can be harmful in rapidly changing environments, preventing generalization of the learned prior to unseen fields. 

The test setup follows from~\cite{BOIL}. Specifically, the prior is meta-trained on miniImageNet with the $4$-block CNN, and meta-tested on tieredImageNet, CUB, and Cars datasets. In addition, $K = 5$ during the meta-training phase, and $K = 10$ in meta-testing to improve adaptation to new tasks from unseen domains. Table~\ref{tab:cross-domain} lists the performance of GD-based, PGD-based, and our MiDA-based approaches. It is seen that MetaMiDA shows superior accuracies on all three datasets. The performance of MetaMiDA is also comparable to the computationally intensive GAP~\cite{GAP} that relies on SVD. This highlights MetaMiDA's remarkable generalization ability despite the large domain gap. Essentially, the loss geometries are related to not only the task data distributions, but also the model structure shared across tasks. Although the domain gap could cause a significant data distribution shift, the learning model remains unchanged across different domains. As a result, domain-generic knowledge can be still captured by the learned loss geometries prior, thus facilitating the learning of cross-domain tasks.

\begin{figure}[t]
\vskip -0.1in
\begin{center}
	\includegraphics[width=1\columnwidth]{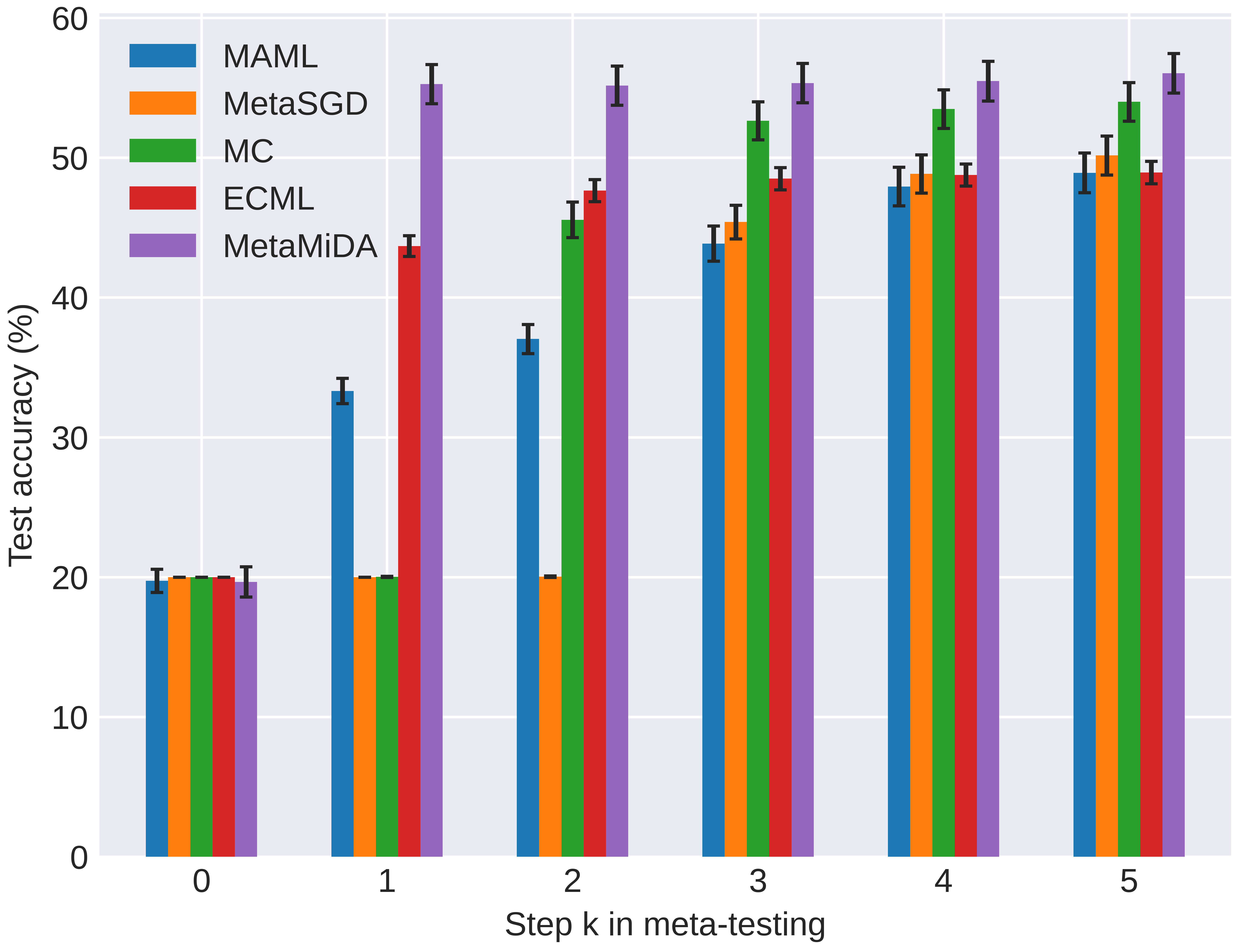}
\end{center}
\vskip -0.2in
\caption{Change of meta-testing accuracy on $5$-class $1$-shot miniImageNet with step $k$. All models are meta-trained and meta-tested with $K = 5$.}
\label{fig:acc-k}
\end{figure}

\subsection{Enhanced performance via accelerated adaptation}
Having confirmed the empirical performance gain of MetaMiDA on popular few-shot classification datasets, the next test analyzes the underlying reason for its superiority. The primary goal of this paper is to learn a more generic loss geometry prior, with which the model can rapidly adapt to new tasks. To illustrate the desired acceleration of task adaptation, Figure~\ref{fig:acc-k} plots the averaged meta-testing accuracy and its $95\%$ confidence interval on $5$-class $1$-shot miniImageNet as a function of $k = 0,1,\ldots,K$. All the methods are meta-trained and meta-tested with $K = 5$, where ECML~\cite{ECML} is a variant of MAML that targets enhanced adaptation convergence by improving the condition number. It is seen that all five meta-learning approaches have comparable accuracies ($\sim 20\%$) at initialization $k=0$. After merely one step, the accuracy of our MetaMiDA increases to over $54\%$, while PGD-based approaches converge more slowly. This is possibly because the quadratic model in PGD-based approaches fails to accurately capture the loss geometries when $k$ is small due to its limited expressivity. In contrast, the generic convex and Lipschitz-smooth model in our MetaMiDA can easily learn this geometric prior, leading to accelerated adaptation. 

Since the computational overhead of meta-learning scales linearly with $K$~\cite{iMAML}, the accelerated adaptation of MetaMiDA implies the possibility of improved scalability via a minimal $K$. In Table~\ref{tab:MetaMiDA-K}, $K$ is reduced from $5$ to $1$ for both meta-training and meta-testing, while the performance of MetaMiDA only deteriorates marginally, yet it remains competitive to state-of-the-art approaches in Table~\ref{tab:comp-PGD}.

\subsection{Complexity and scalability}
This subsection tests the time and space complexities of MetaMiDA in practice. The test is carried out with $5$-class $5$-shot miniImageNet and the $4$-block CNN model. For fairness, the convolutional filters are set to $64$ per block for all methods. 
The results are gathered in Table~\ref{tab:complexity}, where time complexity is gauged relative to MAML, and space complexity is calculated based on the occupied GPU memory. 

Compared to MAML, MetaMiDA with $K=5$ exhibits $14\%$ increased time and $5\%$ increased space complexities. This is slightly inferior to PGD-based methods but superior to the SVD-based approach. The increased complexity stems from the extra gradient computation $\nabla_1 h^*$, and Hessian-vector products involved in backpropagation; cf.~\eqref{eq:app-metagrad-z} and~\eqref{eq:app-metagrad-h}. Fortunately, decreasing $K$ to $1$ can remarkably reduce the computational overhead, thereby enhancing MetaMiDA's scalability. It is noteworthy that the practical time and space complexities are affine rather than linear functions of $K$, due to the constant complexity associated with data reading and storage.

\begin{table}[t]
\vskip -0.15in
\caption{Performance of MetaMiDA with decreased $K$.}
\vskip -0.3in
\label{tab:MetaMiDA-K}
\begin{center}
\begin{tabular}{lcccc}
\toprule
\multirow{2}*{MetaMiDA} & \multicolumn{2}{c}{MiniImageNet, $5$-class} & \multicolumn{2}{c}{TieredImageNet, $5$-class} \\
 & $1$-shot ($\%$) & $5$-shot ($\%$) & $1$-shot ($\%$) & $5$-shot ($\%$) \\
\midrule
$K = 1$ & $55.54_{\pm 1.46}$ & $69.93_{\pm 0.72}$ & $58.96_{\pm 1.43}$ & $74.26_{\pm 0.70}$ \\
$K = 5$ & $56.04_{\pm 1.41}$ & $72.06_{\pm 0.67}$ & $58.52_{\pm 1.45}$ & $75.07_{\pm 0.70}$ \\
\bottomrule
\end{tabular}
\end{center}
\vskip -0.15in
\end{table}

\begin{table}[t]
\caption{Running complexity comparison on $5$-class $5$-shot miniImageNet.}
\vskip -0.2in
\label{tab:complexity}
\begin{center}
\begin{tabular}{lcc}
\toprule
Method & Time (relative) & Space (MB) \\
\midrule
MAML~\cite{MAML} & baseline & $4758$ \\
\hline
MetaSGD~\cite{MetaSGD} & $1.01 \times$ & $4762$ \\
MC~\cite{MetaCurvature} & $1.03 \times$ & $4776$ \\
Approximate GAP~\cite{GAP} & $1.02 \times$ & $4773$ \\
GAP~\cite{GAP} & $1.42 \times$ & $4802$ \\
\hline
MetaMiDA, $K=5$ & $1.14\times$ & $4996$ \\
MetaMiDA, $K=1$ & $0.31\times$ & $3494$ \\
\bottomrule
\end{tabular}
\end{center}
\vskip -0.1in
\end{table}

\section{Concluding remarks}
This contribution has established that the task-level optimization of meta-learning can be enhanced through a versatile loss geometry prior, which captures and optimizes a wide spectrum of loss functions. The proposed MetaMiDA approach generalizes existing PGD-based meta-learning approaches, and offers provable convergence guarantees. Extensive numerical tests on various benchmark datasets illustrate MetaMiDA's superiority in few-shot learning for data-limited applications, cross-domain generalization for rapidly changing environments, and accelerated adaptation with enhanced scalability. 

\ifCLASSOPTIONcompsoc
  \section*{Acknowledgments}
\else
  \section*{Acknowledgment}
\fi
This work was supported by NSF grants 2126052, 2128593, 2212318, 2220292, 2312547, and 2332173. 

\bibliographystyle{IEEEtran}
\bibliography{MetaMiDA_ref}

\newpage

\section*{Biography Section}
\vspace{-1.5in}
\begin{IEEEbiography}[{\includegraphics[width=1in,height=1.25in,clip,keepaspectratio]{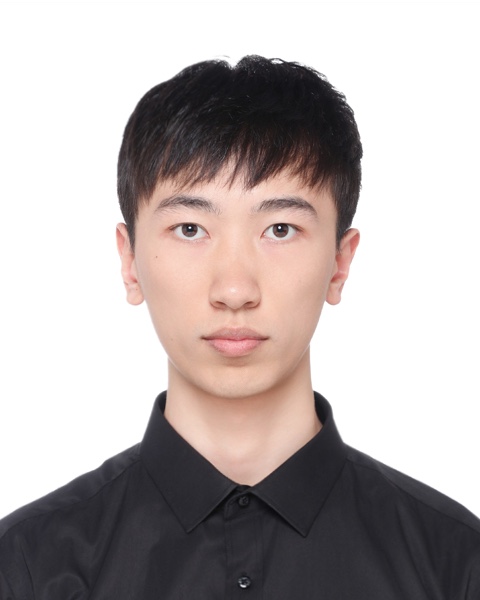}}]{Yilang Zhang} received the B.S. degree in electrical engineering (data science and technology track) from Fudan University, Shanghai, China, in 2020. He is currently working toward the Ph.D. degree in the Department of Electrical and Computer Engineering, University of Minnesota, Minneapolis, USA. His research interests include meta-learning, optimization, and large language models. He is a Member of the SPiNCOM Research Group under the supervision of Prof. Georgios B. Giannakis. He received the National Scholarship from China in 2019, and UMN ADC Graduate Fellowship in 2021.
\end{IEEEbiography}
\vspace{-1.5in}
\begin{IEEEbiography}[{\includegraphics[width=1in,height=1.25in,clip,keepaspectratio]{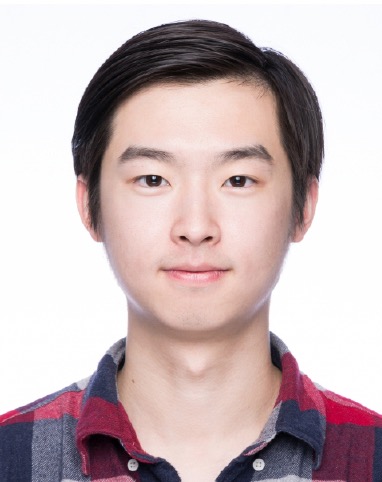}}]{Bingcong Li} received the B.Eng degree (with highest honors) in Information Science and Engineering from Fudan University in 2017, and the Ph.D. degree in Electrical and Computer Engineering from the University of Minnesota in 2022. He is now a Post-Doctoral Research Associate with ETH Zurich, Switzerland. His research interests lie in machine learning, optimization, with application to generalization, robustness and trustworthiness in deep learning and language models. He received the National Scholarship twice from China in 2014 and 2015, and UMN ECE Department Fellowship in 2017. 
\end{IEEEbiography}
\vspace{-1.5in}
\begin{IEEEbiography}[{\includegraphics[width=1in,height=1.25in,clip,keepaspectratio]{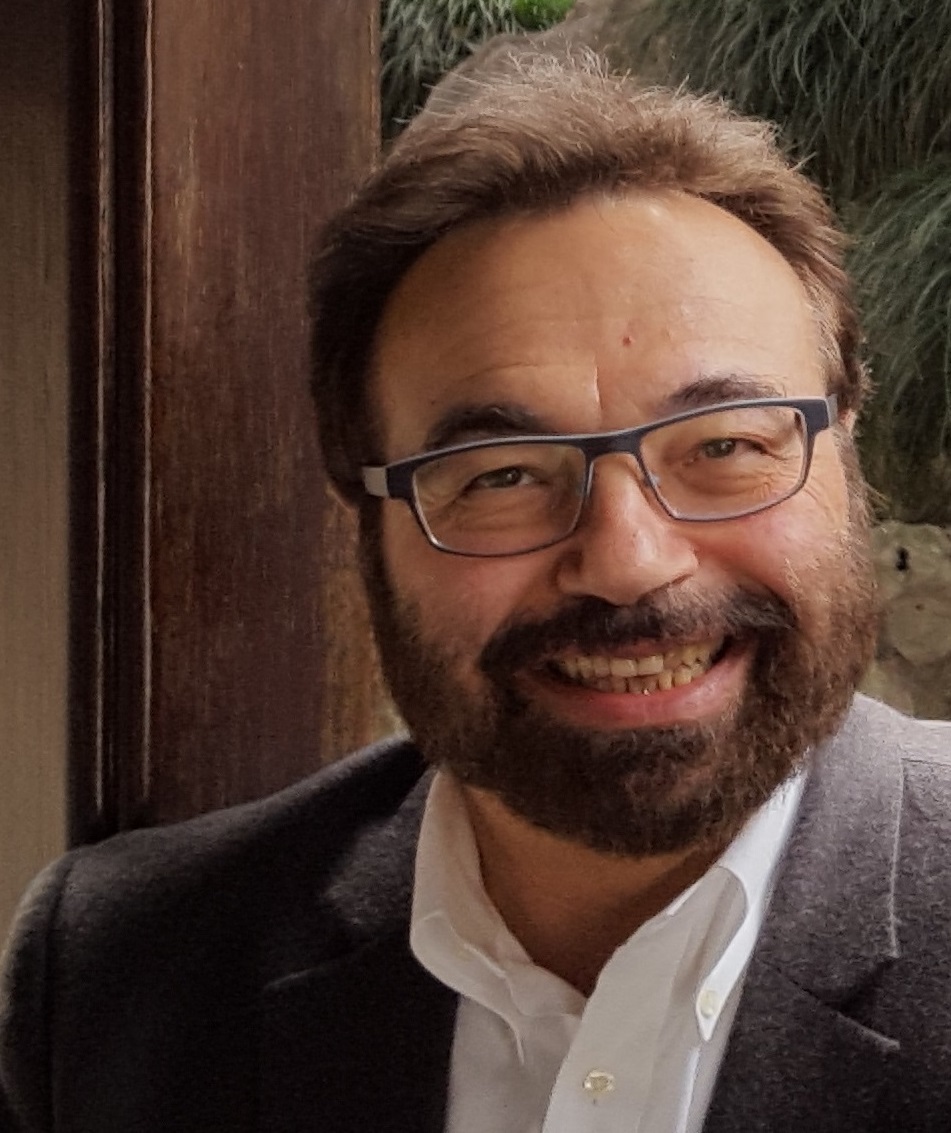}}]{Georgios~B.~Giannakis (F'97)} received his Diploma in Electrical Engr. from the Ntl. Tech. Univ. of Athens, Greece, 1981. From 1982 to 1986 he was with the Univ. of Southern California (USC), where he received his MSc. in Electrical Engineering, 1983, MSc. in Mathematics, 1986, and Ph.D. in Electrical Engr., 1986. He was a faculty member with the University of Virginia from 1987 to 1998, and since 1999 he has been a professor with the Univ. of Minnesota, where he holds an ADC Endowed Chair, a University of Minnesota McKnight Presidential Chair in ECE, and serves as director of the Digital Technology Center. His general interests span the areas of statistical learning, signal processing, communications, and networking - subjects on which he has published more than 480 journal papers, 780 conference papers, 25 book chapters, two edited books and two research monographs. Current research focuses on Data Science, and Network Science with applications to the Internet of Things, and power networks with renewables. He is the (co-) inventor of 34 issued patents, and the (co-) recipient of 10 best journal paper awards from the IEEE Signal Processing (SP) and Communications Societies, including the G. Marconi Prize Paper Award in Wireless Communications. He also received the IEEE-SPS Norbert Wiener Society Award (2019); EURASIP's A. Papoulis Society Award (2020); Technical Achievement Awards from the IEEE-SPS (2000) and from EURASIP (2005); the IEEE ComSoc Education Award (2019); and the IEEE Fourier Technical Field Award (2015). He is a member of the Academia Europaea, and Fellow of the National Academy of Inventors, the European Academy of Sciences, IEEE and EURASIP.
He has served the IEEE in a number of posts, including that of a Distinguished Lecturer for the IEEE-SPS.
\end{IEEEbiography}

%

\clearpage
\appendices
\section{Proof of Theorem~\ref{thm:cvx-Lip-NN} and Corollary~\ref{cor:univ-Lip}}
\label{app:proof-prop-corollary}
\begin{proof}
The proof of Theorem~\ref{thm:cvx-Lip-NN} is inspired by~\cite{ICNN} which studies the convexity of a multi-layer NN, while our analysis further establishes the Lipschitz-smoothness by introducing several additional mild constraints. For notational compactness, denote by $L_\sigma$ and $G_\sigma$ respectively the Lipschitz-continuity and smoothness constants of $\sigma$, and $\| \bfW_i \|_* \le W, \;\forall i$ and $\| \bfM_i \|_* \le M, \;\forall i$ the bounds of NN weights across all layers. 

Theorem~\ref{thm:cvx-Lip-NN} and Corollary~\ref{cor:univ-Lip} can be established via mathematical induction on the total number $I$ of NN layers. First, consider the base case where $I=1$. According to~\eqref{eq:NN-h*}, the NN simplifies to $h^*(\bfz;\bftheta_h) = \bfa_1 = \sigma \big((\bfW_1 + \bfM_1)^\top \bfz + \bfb_1\big)$, where $\bfW_1, \bfM_1 \in \real^d$ reduces to column vectors. Since $\sigma$ is convex and non-decreasing, and pre-activation $f_1(\bfz) := (\bfW_1 + \bfM_1)^\top \bfz + \bfb_1$ is linear thus convex, their composition $h^* = \sigma \circ f_1$ is thereby convex. Furthermore, it follows that
\begin{align*}
	&\| \nabla_1 h^*(\bfz;\bftheta_h) - \nabla_1 h^* (\tilde{\bfz}; \bftheta_h) \|_* \\ 
	&= |\sigma'\big(f_1(\bfz)\big) - \sigma'\big(f_1(\tilde{\bfz})\big)| \| \bfW_1 + \bfM_1 \|_* \\
	&\le G_\sigma | f_1(\bfz) - f_1(\tilde{\bfz}) | \| \bfW_1 + \bfM_1 \|_*  \\
	&= G_\sigma |(\bfW_1 + \bfM_1) (\bfz - \tilde{\bfz})| \| \bfW_1 + \bfM_1 \|_*  \\
	&\overset{(a)}{\le} G_\sigma \| \bfW_1 + \bfM_1 \|_*^2 \| \bfz - \tilde{\bfz} \|  \\
	&\le G_\sigma(W+M)^2 \| \bfz - \tilde{\bfz} \|
\end{align*}
where $(a)$ uses Cauchy-Schwarz inequality. In addition to Lipschitz-smoothness, $h^*$ is also Lipschitz-continuous via
\begin{align*}
	| h^*(\bfz;\bftheta_h) - h^*(\tilde{\bfz};\bftheta_h) | 
	&= |\sigma\big(f_1(\bfz)\big) - \sigma\big(f_1(\tilde{\bfz})\big)| \\
	&\le L_\sigma |f_1(\bfz) - f_1(\tilde{\bfz})| \\
	&= L_\sigma |(\bfW_1 + \bfM_1) (\bfz - \tilde{\bfz})| \\
	&\le L_\sigma \| \bfW_1 + \bfM_1 \|_* \| \bfz - \tilde{\bfz} \| \\
	&\le L_\sigma (W+M) \| \bfz - \tilde{\bfz} \|.
\end{align*}
It is worth noting that the continuity and smoothness constants $L_\sigma (W+M)$ and $G_\sigma (W+M)^2$ neither rely on $\bftheta_h$. 

Now assume with any $\bftheta_h$ satisfying Theorem~\ref{thm:cvx-Lip-NN}, $h^*$ is convex, Lipschitz-continuous, and Lipschitz-smooth wrt $\bfz$ for $I=1,\ldots,I_0$. We next show that the same properties also hold for $I=I_0+1$. For the case $I=I_0+1$,~\eqref{eq:NN-h*} implies $h^*(\bfz; \bftheta_h) = \bfa_{I_0+1} = \sigma(\bfW_{I_0+1}^\top \bfa_{I_0} + \bfM_{I_0+1}^\top \bfz + \bfb_{I_0})$, where $\bfW_{I_0+1}, \bfM_{I_0+1} \in \real^d$ boils down to vectors. As each entry $[\bfa_{I_0}]_j, \forall j$ is an $I_0$-layer NN defined by~\eqref{eq:NN-h*}, the inductive hypothesis implies $[\bfa_{I_0}]_j$ is convex, Lipschitz-continuous, and Lipschitz-smooth wrt $\bfz$ regardless of $\bftheta_h$. For later use, let $L_{h,I_0}$ and $G_{h,I_0}$ be the maximum Lipschitz-continuity and smoothness constants for $[\bfa_{I_0}]_j$ across the index $j$. Notice that since $\bfW_{I_0+1}$ is element-wise non-negative, $\bfW_{I_0+1}^\top \bfa_{I_0}$ is a conical combination of functions convex in $\bfz$, and is thus also convex in $\bfz$. Then, it is straightforward to see that the pre-activation $f_{I_0+1} (\bfz) := \bfW_{I_0+1}^\top \bfa_{I_0} + \bfM_{I_0+1}^\top \bfz + \bfb_{I_0}$ is convex. Hence, the composition $h^* = \sigma \circ f_{I_0+1}$ is convex. 

Next, we conclude the induction by proving that $h^*$ is Lipschitz-continuous and smooth with $I=I_0+1$. For $\forall \bfz, \tilde{\bfz} \in \real^d$, let $\tilde{\bfa}_i$ be the hidden neurons of the $i$-th layer when $\tilde{\bfz}$ serves as the input of $h^*$. On one hand,
\begin{align}
\label{eq:app-lemma-Lip}
	&\| \nabla_1 h^*(\bfz;\bftheta_h) - \nabla_1 h^* (\tilde{\bfz}; \bftheta_h) \|_* \nonumber \\ 
	&= \| \sigma'(f_{I_0+1}(\bfz))  \big( \nabla_\bfz \bfa_{I_0} \bfW_{I_0+1}  + \bfM_{I_0+1} \big) - \nonumber \\
	&\qquad \sigma'(f_{I_0+1}(\tilde{\bfz})) \big( \nabla_\bfz \tilde{\bfa}_{I_0} \bfW_{I_0+1}  + \bfM_{I_0+1} \big) \|_* \nonumber \\
	&\le \|\sigma'(f_{I_0+1}(\bfz)) \nabla_\bfz \bfa_{I_0}  - \sigma'\big(f_{I_0+1}(\tilde{\bfz})\big) \nabla_\bfz \tilde{\bfa}_{I_0} \|_*  \| \bfW_{I_0+1} \|_* + \nonumber \\
	&\quad\;\, \| \sigma'(f_{I_0+1}(\bfz)) - \sigma'\big(f_{I_0+1}(\tilde{\bfz})\big) \|_* \| \bfM_{I_0+1} \|_* \nonumber \\
	&\le W |\sigma'(f_{I_0+1}(\bfz))| \times \| \nabla_\bfz \bfa_{I_0} - \nabla_\bfz \tilde{\bfa}_{I_0} \|_* + \nonumber \\
	&\quad\;\, W | \sigma'(f_{I_0+1}(\bfz)) - \sigma'(f_{I_0+1}(\tilde{\bfz})) | \times \| \nabla_\bfz \tilde{\bfa}_{I_0} \| + \nonumber \\
	&\quad\;\, MG_\sigma | f_{I_0+1}(\bfz) - f_{I_0+1}(\tilde{\bfz}) | \nonumber \\
	&\le W L_\sigma G_{h,I_0} \| \bfz - \tilde{\bfz} \| + W L_{h,I_0} G_\sigma | f_{I_0+1}(\bfz) - f_{I_0+1}(\tilde{\bfz}) | + \nonumber \\
	&\quad\;\, MG_\sigma | f_{I_0+1}(\bfz) - f_{I_0+1}(\tilde{\bfz}) |.
\end{align}
Using Cauchy-Schwarz inequality yields
\begin{align}
\label{eq:app-lemma-preact}
	&| f_{I_0+1}(\bfz) - f_{I_0+1}(\tilde{\bfz}) | \nonumber \\
	&= | \bfW_{I_0+1}^\top (\bfa_{I_0} - \tilde{\bfa}_{I_0}) + \bfM_{I_0+1}^\top (\bfz - \tilde{\bfz}) | \nonumber \\
	&\le \| \bfW_{I_0+1} \|_* \| \bfa_{I_0} - \tilde{\bfa}_{I_0} \| + \| \bfM_{I_0+1} \|_* \| \bfz - \tilde{\bfz}\|_2 \nonumber \\
	&\le ( WL_{h,I_0} + M ) \| \bfz - \tilde{\bfz} \|. 
\end{align}
Combining~\eqref{eq:app-lemma-Lip} with~\eqref{eq:app-lemma-preact} renders the Lipschitz-smoothness
\begin{equation*}
	\| \nabla_1 h^*(\bfz;\bftheta_h) - \nabla_1 h^* (\tilde{\bfz}; \bftheta_h) \|_* \le G_{h,I_0+1} \| \bfz - \tilde{\bfz} \|
\end{equation*}
where $G_{h,I_0+1} := W L_\sigma G_{h,I_0} + G_\sigma (W L_{h,I_0} + M ) (W L_{h,I_0} + M)$. 
On the other hand, the Lipschitz-continuity holds via
\begin{align*}
	| h^*(\bfz;\bftheta_h) - h^*(\tilde{\bfz};\bftheta_h) | 
	&= | \sigma(f_{I_0+1}(\bfz)) - \sigma(f_{I_0+1}(\tilde{\bfz})) | \\
	&\le L_\sigma |f_{I_0+1}(\bfz) - f_{I_0+1}(\tilde{\bfz})| \\
	&\overset{(a)}{\le} L_\sigma (WL_{h,I_0} + M) \| \bfz - \tilde{\bfz} \| \\
	&:= L_{h,I_0+1} \| \bfz - \tilde{\bfz} \|
\end{align*}
where $(a)$ follows from~\eqref{eq:app-lemma-preact}. 

Finally, it is easy to see that $L_{h,I_0+1}$ and $G_{h,I_0+1}$ are independent of $\bftheta_h$ because the inductive hypothesis implies that all the constants in their definitions are universal wrt $\bftheta_h$. This concludes the induction. 
\end{proof}

\section{Proof of Proposition~\ref{prop:meta-smooth}}
We first prove~\eqref{eq:meta-smooth-z}. Defining primal variable $\bfphi_t^k := \nabla_1 h^*(\bfz_t^k;\bftheta_h)$, it follows from the chain rule that
\begin{align}
\label{eq:app-metagrad-z}
	\nabla_1 \metaloss (\bftheta) 
	&= \expect_t \nabla_1 \metaloss_t (\bftheta_z,\bftheta_h) \nonumber \\
	&= \expect_t \bigg\{ \prod_{k=0}^{K-1} \big[  \nabla_{\bfz_t^k} \bfz_t^{k+1} \big]  \nabla_1 (\loss{val}_t \circ \nabla_1 h^*) (\bfz_t^K;\bftheta_h) \bigg\} \nonumber \\
	&= \expect_t \bigg\{ \prod_{k=0}^{K-1} \big[ \bfI_d - \alpha \nabla_1 (\nabla \loss{trn}_t \circ \nabla_1 h^*) (\bfz_t^k; \bftheta_h) \big] \times \nonumber \\
	&\qquad\quad\; \nabla_1^2 h^*(\bfz_t^K;\bftheta_h) \nabla \loss{val}_t (\bfphi_t^K) \bigg\} \nonumber \\
	&:= \expect_t \bigg\{ \prod_{k=0}^{K-1} \big[ \bfI_d - \alpha \bfG_t^k \big] \bfg_t^K \bigg\} 
\end{align}
where $\bfG_t^k := \nabla_1 (\nabla \loss{trn}_t \circ \nabla_1 h^*) (\bfz_t^k; \bftheta_h) = \nabla_1^2 h^*(\bfz_t^k;\bftheta_h) \allowbreak \nabla^2\loss{trn}_t(\bfphi_t^k)$. 

Let $\tilde{\bfz}_t^k$ and $\tilde{\bfphi}_t^k$ be the dual and primal variables of the $k$-th iteration incurred by meta-parameter $\tilde{\bftheta} \in \real^D$, $\tilde{\bfG}_t^k := \nabla_1^2 h^*(\tilde{\bfz}_t^k;\tilde{\bftheta}_h) \nabla^2\loss{trn}_t(\tilde{\bfphi}_t^k)$, and $\tilde{\bfg}_t^k := \nabla_1^2 h^*(\tilde{\bfz}_t^K;\tilde{\bftheta}_h) \allowbreak \nabla \loss{val}_t (\tilde{\bfphi}_t^K)$. The per-task smoothness wrt $\bftheta_z$ is 
\begin{align}
\label{eq:app-meta-smooth-z}
	&\| \nabla_1 \metaloss_t(\bftheta) - \nabla_1 \metaloss_t(\tilde{\bftheta}) \|_2 \nonumber \\
	&\le \bigg\| \prod_{k=0}^{K-1} \big[ \bfI_d - \alpha \bfG_t^k \big] \bfg_t^K - \prod_{k=0}^{K-1} \big[ \bfI_d - \alpha \tilde{\bfG}_t^k \big] \bfg_t^K \bigg\|_2 + \nonumber \\
	&\quad\;\, \bigg\| \prod_{k=0}^{K-1} \big[ \bfI_d - \alpha \tilde{\bfG}_t^k \big] \bfg_t^K - \prod_{k=0}^{K-1} \big[ \bfI_d - \alpha \tilde{\bfG}_t^k \big] \tilde{\bfg}_t^K \bigg\|_2 \nonumber \\
	&\le \bigg\| \prod_{k=0}^{K-1} \big[ \bfI_d - \alpha \bfG_t^k \big] - \prod_{k=0}^{K-1} \big[ \bfI_d - \alpha \tilde{\bfG}_t^k \big] \bigg\|_2 \| \bfg_t^K \|_2 + \nonumber \\
	&\quad\;\, \prod_{k=0}^{K-1} \| \bfI_d - \alpha \tilde{\bfG}_t^k \|_2  \| \bfg_t^K - \tilde{\bfg}_t^K \|_2. 
\end{align}

From Assumptions~\ref{as:loss-fn}.ii) and~\ref{as:inv-mirror-map}.i), it holds that
\begin{align}
\label{eq:app-gamma}
	\| \bfI_d - \alpha \tilde{\bfG}_t^k \|_2 &\le 1 + \alpha \| \nabla_1^2 h^*(\tilde{\bfz}_t^k;\tilde{\bftheta}_h) \|_2 \| \nabla^2\loss{trn}_t(\tilde{\bfphi}_t^k) \|_2 \nonumber \\
	&\le 1 + \alpha G_h G_\loss{} = \gamma. 
\end{align}
The same bound holds for $\| \bfI_d - \alpha \bfG_t^k \|_2,\forall k$, hence we get
\begin{align}
\label{eq:app-prodIaG-Lip}
	&\bigg\| \prod_{k=0}^{K-1} \big[ \bfI_d - \alpha \bfG_t^k \big] - \prod_{k=0}^{K-1} \big[ \bfI_d - \alpha \tilde{\bfG}_t^k \big] \bigg\|_2 \nonumber \\
	&\le \bigg\| \prod_{k=0}^{K-1} \big[ \bfI_d - \alpha \bfG_t^k \big] - \prod_{k=0}^{K-2} \big[ \bfI_d - \alpha \tilde{\bfG}_t^k \big] (\bfI_d - \alpha \bfG_t^{K-1}) \bigg\|_2 + \nonumber \\
	&\quad\;\, \bigg\| \prod_{k=0}^{K-2} \big[ \bfI_d - \alpha \tilde{\bfG}_t^k \big] (\bfI_d - \alpha \bfG_t^{K-1}) - \prod_{k=0}^{K-1} \big[ \bfI_d - \alpha \tilde{\bfG}_t^k \big] \bigg\|_2 \nonumber \\
	&\le \bigg\| \prod_{k=0}^{K-2} \big[ \bfI_d - \alpha \bfG_t^k \big] - \prod_{k=0}^{K-2} \big[ \bfI_d - \alpha \tilde{\bfG}_t^k \big] \bigg\|_2 \| \bfI_d - \alpha \bfG_t^{K-1} \|_2 + \nonumber \\
	&\quad\;\, \prod_{k=0}^{K-2} \big[ \| \bfI_d - \alpha \tilde{\bfG}_t^k \|_2 \big] \| \alpha ( \bfG_t^{K-1} - \tilde{\bfG}_t^{K-1} ) \|_2 \nonumber \\
	&\overset{(a)}{\le} \gamma \bigg\| \prod_{k=0}^{K-2} \big[ \bfI_d - \alpha \bfG_t^k \big] - \prod_{k=0}^{K-2} \big[ \bfI_d - \alpha \tilde{\bfG}_t^k \big] \bigg\|_2 + \nonumber \\
	&\quad\;\, \alpha \gamma^{K-1} \| \bfG_t^{K-1} - \tilde{\bfG}_t^{K-1} \|_2 \nonumber \\
	&\overset{(b)}{\le} \alpha \gamma^{K-1} \sum_{k=0}^{K-1} \| \bfG_t^k - \tilde{\bfG}_t^k \|_2
\end{align}
where $(a)$ leverages~\eqref{eq:app-gamma}, and $(b)$ is by telescoping. 

Then applying Assumptions~\ref{as:loss-fn} and~\ref{as:inv-mirror-map}, we can bound
\begin{align}
\label{eq:app-Gtk-Lip}
	&\| \bfG_t^k - \tilde{\bfG}_t^k \|_2 \nonumber \\
	&\le \| \bfG_t^k - \nabla_1^2 h^*(\bfz_t^k;\bftheta_h) \nabla^2\loss{trn}_t(\tilde{\bfphi}_t^k) \|_2 + \nonumber \\
	&\quad\;\, \| \nabla_1^2 h^*(\bfz_t^k;\bftheta_h) \nabla^2\loss{trn}_t(\tilde{\bfphi}_t^k) - \tilde{\bfG}_t^k \|_2 \nonumber \\
	&\le \| \nabla_1^2 h^*(\bfz_t^k;\bftheta_h) \|_2 \| \nabla^2\loss{trn}_t(\bfphi_t^k) - \nabla^2\loss{trn}_t(\tilde{\bfphi}_t^k) \|_2 + \nonumber \\
	&\quad\;\, \| \nabla_1^2 h^*(\bfz_t^k;\bftheta_h) - \nabla_1^2 h^*(\tilde{\bfz}_t^k; \tilde{\bftheta}_h) \|_2 \| \nabla^2\loss{trn}_t(\tilde{\bfphi}_t^k) \|_2 \nonumber \\
	&\le G_h H_\loss{} \| \bfphi_t^k - \tilde{\bfphi}_t^k \|_2 + G_\loss{}h_h (\| \bfz_t^k - \tilde{\bfz}_t^k \|_2 + \| \bftheta_h - \tilde{\bftheta}_h \|_2) \nonumber \\
	&= G_h H_\loss{} \| \nabla_1 h^*(\bfz_t^k; \bftheta) - \nabla_1 h^*(\tilde{\bfz}_t^k; \tilde{\bftheta}) \|_2 + \nonumber \\
	&\quad\;\, G_\loss{}h_h (\| \bfz_t^k - \tilde{\bfz}_t^k \|_2 + \| \bftheta_h - \tilde{\bftheta}_h \|_2) \nonumber \\
	&\le (G_h^2 H_\loss{} + G_\loss{}h_h) (\| \bfz_t^k - \tilde{\bfz}_t^k \|_2 + \| \bftheta_h - \tilde{\bftheta}_h \|_2) \nonumber \\
	&\le \gamma^k (G_h^2 H_\loss{} + G_\loss{}h_h) \| \bftheta - \tilde{\bftheta} \|_2
\end{align}
where the last inequality leverages Lemma~\ref{lemma:ztk-Lip}. 

Relating~\eqref{eq:app-prodIaG-Lip} to~\eqref{eq:app-Gtk-Lip} leads to
\begin{align}
\label{eq:app-prodIaG-Lip-bound}
	&\bigg\| \prod_{k=0}^{K-1} \big[ \bfI_d - \alpha \bfG_t^k \big] - \prod_{k=0}^{K-1} \big[ \bfI_d - \alpha \tilde{\bfG}_t^k \big] \bigg\|_2 \nonumber \\
	&\le \gamma^{K-1} (\gamma^K - 1) (G_h\frac{H_\loss{}}{G_\loss{}} + \frac{H_h}{G_h}).
\end{align}

Next, we bound the term 
\begin{align}
\label{eq:app-gtK-Lip}
	&\| \bfg_t^K - \tilde{\bfg}_t^K \|_2 \nonumber\\
	&= \| \nabla_1 (\loss{val}_t \circ \nabla_1 h^*) (\bfz_t^K;\bftheta_h) - \nabla_1 (\loss{val}_t \circ \nabla_1 h^*) (\tilde{\bfz}_t^K; \tilde{\bftheta}_h) \|_2 \nonumber \\ 
	&\overset{(a)}{\le} G_{\loss{}h} (\| \bfz_t^k - \tilde{\bfz}_t^K \|_2 + \| \bftheta_h - \tilde{\bftheta}_h \|_2 ) \nonumber \\
	&\overset{(b)}{\le} \gamma^K G_{\loss{}h} \| \bftheta - \tilde{\bftheta} \|_2
\end{align}
where $(a)$ and $(b)$ are respectively due to Assumption~\ref{as:loss-fn}.iv) and Lemma~\ref{lemma:ztk-Lip}. 

Now, combing~\eqref{eq:app-meta-smooth-z} with~\eqref{eq:app-gamma},~\eqref{eq:app-prodIaG-Lip-bound},~\eqref{eq:app-gtK-Lip}, and the definition of $C_{G, 1}$ results in
\begin{equation*}
	\| \nabla_1 \metaloss_t(\bftheta) - \nabla_1 \metaloss_t(\tilde{\bftheta}) \|_2 \le \big( C_{G,1} \| \bfg_t^K \|_2 + \gamma^{2K} G_{\loss{}h} \big) \| \bftheta - \tilde{\bftheta} \|_2. 
\end{equation*}
Taking $\expect_t$ on both side and using Jensen's inequality to exchange $\expect_t$ with the left $\| \cdot \|_2$ finish the proof of~\eqref{eq:meta-smooth-z}. 

To wrap up the proof, we next show~\eqref{eq:meta-smooth-h}. Similarly, from the chain rule we can obtain
\begin{align}
\label{eq:app-metagrad-h}
	&\nabla_2 \metaloss (\bftheta) = \expect_t \nabla_2 \metaloss_t (\bftheta_z, \bftheta_h) \nonumber\\
    &= \expect_t \big\{ \nabla_{\bftheta_h} \bfz_t^K \nabla_1 (\loss{val}_t \circ \nabla_1 h^*) (\bfz_t^K; \bftheta_h) + \nonumber \\
	&\quad\;\,\quad \nabla_2 (\loss{val}_t \circ \nabla_1 h^*) (\bfz_t^K; \bftheta_h) \big\} \nonumber\\
	&= \expect_t \big\{ \nabla_{\bftheta_h} \bfz_t^{K-1} \nabla_{\bfz_t^{K-1}} \bfz_t^K \bfg_t^K +  \nabla_2 (\loss{val}_t \circ \nabla_1 h^*) (\bfz_t^K; \bftheta_h) \big\} \nonumber \\
	&\overset{(a)}{=} \expect_t \bigg\{ \sum_{k=0}^{K-1}  \Big[ \nabla_{\bftheta_h} \bfz_t^{k+1} \prod_{l=k+1}^{K-1} \nabla_{\bfz_t^l} \bfz_t^{l+1} \Big] \bfg_t^K + \nonumber \\
	&\quad\;\,\quad \nabla_2 (\loss{val}_t \circ \nabla_1 h^*) (\bfz_t^K; \bftheta_h) \big\} \nonumber \\
	&\overset{(b)}{=} \expect_t \bigg\{ \sum_{k=0}^{K-1} \Big[ -\alpha \nabla_2 (\nabla \loss{trn}_t \circ \nabla_1 h^*) (\bfz_t^k; \bftheta_h) \times \nonumber \\
	&\quad\;\,\quad \prod_{l=k+1}^{K-1} (\bfI_d - \alpha \bfG_t^l) \Big] \bfg_t^K + \nabla_2 (\loss{val}_t \circ \nabla_1 h^*) (\bfz_t^K; \bftheta_h) \bigg\} \nonumber \\
	&:= \expect_t \bigg\{ -\alpha \sum_{k=0}^{K-1} \bfH_t^k \bfg_t^K + \bfh_t^K \bigg\}
\end{align}
where $(a)$ is by recursively applying the third equation for $K-1$ times, $(b)$ utilizes~\eqref{eq:MetaMDA-obj-task}, $\bfh_t^K := \nabla_2 (\loss{val}_t \circ \nabla_1 h^*) (\bfz_t^K; \bftheta_h)$, and $\bfH_t^k := \nabla_2 (\nabla \loss{trn}_t \circ \nabla_1 h^*) (\bfz_t^k; \bftheta_h) \prod_{l=k+1}^{K-1} (\bfI_d - \alpha \bfG_t^l)$. 

Likewise, define $\tilde{\bfH}_t^k$ and $\tilde{\bfh}_t^K$ the corresponding matrix and vector rendered by initialization $\tilde{\bftheta}_z$. The per-task smoothness wrt $\bftheta_h$ is 
\begin{align}
\label{eq:app-meta-smooth-h}
	&\| \nabla_1 \metaloss_t(\bftheta) - \nabla_1 \metaloss_t(\tilde{\bftheta}) \|_2 \nonumber \\
	&\le \alpha \sum_{k=0}^{K-1} \big[ \| \bfH_t^k \bfg_t^K - \tilde{\bfH}_t^k \tilde{\bfg}_t^K \|_2 \big] + \| \bfh_t^K - \tilde{\bfh}_t^K \|_2 \nonumber \\
	&\le \alpha \sum_{k=0}^{K-1} \big[ \| \bfH_t^k \bfg_t^K - \tilde{\bfH}_t^k \bfg_t^K \|_2 + \| \tilde{\bfH}_t^k \bfg_t^K - \tilde{\bfH}_t^k \tilde{\bfg}_t^K \|_2 \big] + \nonumber \\
	&\quad\;\, \| \bfh_t^K - \tilde{\bfh}_t^K \|_2 \nonumber \\
	&\le \alpha \sum_{k=0}^{K-1} \big[ \| \bfH_t^k - \tilde{\bfH}_t^k \|_2 \| \bfg_t^K \|_2 + \| \tilde{\bfH}_t^k \|_2 \| \bfg_t^K - \tilde{\bfg}_t^K \|_2 \big] + \nonumber \\
	&\quad\;\, \| \bfh_t^K - \tilde{\bfh}_t^K \|_2. 
\end{align}

Notice that
\begin{align}
\label{eq:app-Htk-smooth}
	&\| \bfH_t^k - \tilde{\bfH}_t^k \|_2 \nonumber \\
	&\le \bigg\| \bfH_t^k - \nabla_2 (\nabla \loss{trn}_t \circ \nabla_1 h^*) (\tilde{\bfz}_t^k; \tilde{\bftheta}_h) \prod_{l=k+1}^{K-1} (\bfI_d - \alpha \bfG_t^l) \bigg\|_2 + \nonumber \\
	&\quad\;\, \bigg\| \nabla_2 (\nabla \loss{trn}_t \circ \nabla_1 h^*) (\tilde{\bfz}_t^k; \tilde{\bftheta}_h) \prod_{l=k+1}^{K-1} (\bfI_d - \alpha \bfG_t^l) - \tilde{\bfH}_t^k \bigg\|_2 \nonumber \\
	&\le \| \nabla_2 (\nabla  \loss{trn}_t \circ \nabla_1 h^*) (\bfz_t^k; \bftheta_h) - \nabla_2 (\nabla \loss{trn}_t \circ \nabla_1 h^*) (\tilde{\bfz}_t^k; \tilde{\bftheta}_h) \|_2 \times \nonumber \\
	&\quad\;\, \prod_{l=k+1}^{K-1} \big\| \bfI_d - \alpha \bfG_t^l \big\|_2 + \| \nabla_2 (\nabla \loss{trn}_t \circ \nabla_1 h^*) (\tilde{\bfz}_t^k; \tilde{\bftheta}_h) \|_2 \times \nonumber \\
	&\quad\;\, \bigg\| \prod_{l=k+1}^{K-1} (\bfI_d - \alpha \bfG_t^l) - \prod_{l=k+1}^{K-1} (\bfI_d - \alpha \tilde{\bfG}_t^l) \bigg\|_2. 
\end{align}

On one hand, we have
\begin{align}
\label{eq:app-mixJac-Lip}
	&\| \nabla_2 (\nabla \loss{trn}_t \circ \nabla_1 h^*) (\bfz_t^k; \bftheta_h) - \nabla_2 (\nabla \loss{trn}_t \circ \nabla_1 h^*) (\tilde{\bfz}_t^k; \tilde{\bftheta}_h) \|_2 \nonumber \\
	&= \| \nabla_2 \nabla_1 h^*(\bfz_t^k; \bftheta_h) \nabla^2 \loss{trn}_t (\bfphi_t^k) - \nonumber
	 \\
	&\quad\;\, \nabla_2 \nabla_1 h^*(\tilde{\bfz}_t^k; \tilde{\bftheta}_h) \nabla^2 \loss{trn}_t (\tilde{\bfphi}_t^k) \|_2 \nonumber \\
	&\le \| \nabla_2 \nabla_1 h^*(\bfz_t^k; \bftheta_h) - \nabla_2 \nabla_1 h^*(\tilde{\bfz}_t^k; \tilde{\bftheta}_h) \|_2 \| \nabla^2 \loss{trn}_t (\bfphi_t^k) \|_2 + \nonumber
	 \\
	&\quad\;\, \| \nabla_2 \nabla_1 h^*(\tilde{\bfz}_t^k; \tilde{\bftheta}_h) \|_2 \| \nabla^2 \loss{trn}_t (\bfphi_t^k) - \nabla^2 \loss{trn}_t (\tilde{\bfphi}_t^k) \|_2 \nonumber \\
	&\overset{(a)}{\le} H_h G_\loss{} (\| \bfz_t^k - \tilde{\bfz}_t^k \|_2 + \| \bftheta_h - \tilde{\bftheta}_t \|_2) + \nonumber \\
	&\quad\;\, G_h H_\loss{} \| \nabla_1 h^* (\bfz_t^k; \bftheta_h) - \nabla_1 h^* (\tilde{\bfz}_t^k; \tilde{\bftheta}_h) \| _2 \nonumber \\
	&\overset{(b)}{\le} (G_\loss{}h_h + G_h^2 H_\loss{}) (\| \bfz_t^k - \tilde{\bfz}_t^k \|_2 + \| \bftheta_h - \tilde{\bftheta}_t \|_2) \nonumber \\
	&\overset{(c)}{\le} \gamma^k (G_\loss{}h_h + G_h^2 H_\loss{}) \| \bftheta - \tilde{\bftheta} \|_2
\end{align}
where $(a)$ and $(b)$ use Assumptions~\ref{as:loss-fn} and~\ref{as:inv-mirror-map}, and $(b)$ relies on Theorem~\ref{thm:cvx-Lip-NN}. 

On the other hand, using Assumptions~\ref{as:loss-fn} and~\ref{as:inv-mirror-map} gives 
\begin{align}
\label{eq:app-mixJac-norm}
	&\| \nabla_2 (\nabla \loss{trn}_t \circ \nabla_1 h^*) (\tilde{\bfz}_t^k; \tilde{\bftheta}_h) \|_2 \nonumber \\
	&= \| \nabla_2 \nabla_1 h^*(\tilde{\bfz}_t^k; \tilde{\bftheta}_h) \nabla^2 \loss{trn}_t (\tilde{\bfphi}_t^k) \|_2 \nonumber \\
	&\le \| \nabla_2 \nabla_1 h^*(\tilde{\bfz}_t^k; \tilde{\bftheta}_h) \|_2 \| \nabla^2 \loss{trn}_t (\tilde{\bfphi}_t^k) \|_2 \le G_h G_\loss{}
\end{align}
and
\begin{align}
\label{eq:app-k+1prodIaG-Lip}
	&\bigg\| \prod_{l=k+1}^{K-1} (\bfI_d - \alpha \bfG_t^l) - \prod_{l=k+1}^{K-1} (\bfI_d - \alpha \tilde{\bfG}_t^l) \bigg\|_2 \nonumber \\
	&\le \bigg\| \prod_{l=k+1}^{K-1} (\bfI_d - \alpha \bfG_t^l) - (\bfI_d - \alpha \tilde{\bfG}_t^{k+1}) \prod_{l=k+2}^{K-1} (\bfI_d - \alpha \bfG_t^l) \bigg\|_2 + \nonumber \\
	&\quad\;\, \bigg\| (\bfI_d - \alpha \tilde{\bfG}_t^{k+1}) \prod_{l=k+2}^{K-1} (\bfI_d - \alpha \bfG_t^l) - \prod_{l=k+1}^{K-1} (\bfI_d - \alpha \tilde{\bfG}_t^l) \bigg\|_2 \nonumber \\
	&\le \big\| \alpha (\bfG_t^{k+1} - \tilde{\bfG}_t^{k+1}) \big\|_2 \bigg\| \prod_{l=k+2}^{K-1} (\bfI_d - \alpha \bfG_t^l) \bigg\|_2 + \nonumber \\
	&\quad\;\, \big\| \bfI_d - \alpha \tilde{\bfG}_t^{k+1} \big\|_2 \bigg\| \prod_{l=k+2}^{K-1} (\bfI_d - \alpha \bfG_t^l) - \prod_{l=k+2}^{K-1} (\bfI_d - \alpha \tilde{\bfG}_t^l) \bigg\|_2 \nonumber \\
	&\overset{(a)}{\le} \alpha \gamma^{k+1} (G_h^2 H_\loss{} + G_\loss{}h_h) \| \bftheta - \tilde{\bftheta} \|_2 \times \gamma^{K-k-2} + \nonumber \\
	&\quad\;\, \gamma \bigg\| \prod_{l=k+2}^{K-1} (\bfI_d - \alpha \bfG_t^l) - \prod_{l=k+2}^{K-1} (\bfI_d - \alpha \tilde{\bfG}_t^l) \bigg\|_2 \nonumber \\
	&\overset{(b)}{\le} \alpha \gamma^{K-1} (G_h^2 H_\loss{} + G_\loss{}h_h) \| \bftheta - \tilde{\bftheta} \|_2 \sum_{l=0}^{K-k-2} \gamma^l \nonumber \\
	&= \gamma^{K-1} (\gamma^{K-k-1} - 1) (G_h \frac{H_\loss{}}{G_\loss{}} + \frac{H_h}{G_h}) \| \bftheta - \tilde{\bftheta} \|_2
\end{align}
where $(a)$ is from~\eqref{eq:app-gamma} and~\eqref{eq:app-Gtk-Lip}, and $(b)$ telescopes $(a)$. 

Now, plug~\eqref{eq:app-gamma},~\eqref{eq:app-mixJac-Lip},~\eqref{eq:app-mixJac-norm}, and~\eqref{eq:app-k+1prodIaG-Lip} into~\eqref{eq:app-Htk-smooth} to yield
\begin{align}
\label{eq:app-Htk-smooth-bound}
	\| \bfH_t^k - \tilde{\bfH}_t^k \|_2 \le L_{H,k} \| \bftheta - \tilde{\bftheta} \|_2
\end{align}
where $L_{H,k} := \gamma^{K-1} [ G_h^2 H_\loss{} + G_\loss{}h_h + (\gamma^{K-k-1} - 1) (G_h^2 H_\loss{} + G_\loss{}h_h) ]$. 

Utilizing~\eqref{eq:app-mixJac-norm} and~\eqref{eq:app-gamma} renders
\begin{align}
\label{eq:app-Htk-norm}
	\| \tilde{\bfH}_t^k \|_2 &= \bigg\| \nabla_2 (\nabla \loss{trn}_t \circ \nabla_1 h^*) (\tilde{\bfz}_t^k; \tilde{\bftheta}_h) \prod_{l=k+1}^{K-1} (\bfI_d - \alpha \tilde{\bfG}_t^l) \bigg\|_2 \nonumber \\
	&\le \| \nabla_2 (\nabla \loss{trn}_t \circ \nabla_1 h^*) (\tilde{\bfz}_t^k; \tilde{\bftheta}_h) \|_2 \prod_{l=k+1}^{K-1} \| \bfI_d - \alpha \tilde{\bfG}_t^l \|_2 \nonumber \\
	&\le \gamma^{K-k-1} G_h G_\loss{}. 
\end{align}

Moreover, Assumption~\ref{as:loss-fn}.iv) suggests
\begin{align}
\label{eq:app-htK-smooth}
	&\| \bfh_t^K - \tilde{\bfh}_t^K \|_2 \nonumber \\
	&= \| \nabla_2 (\loss{val}_t \circ \nabla_1 h^*) (\bfz_t^K; \bftheta_h) - \nabla_2 (\loss{val}_t \circ \nabla_1 h^*) (\tilde{\bfz}_t^K; \tilde{\bftheta}_h) \|_2 \nonumber \\
	&\le G_{\loss{}h} (\| \bfz_t^K - \tilde{\bfz}_t^K \|_2 + \| \bftheta_h - \tilde{\bftheta}_h \|_2) \nonumber \\
	&\le \gamma^K G_{\loss{}h} \| \bftheta - \tilde{\bftheta} \|_2
\end{align}
with the last inequality following from Lemma~\ref{lemma:ztk-Lip}. 

Plugging~\eqref{eq:app-gtK-Lip},~\eqref{eq:app-Htk-smooth-bound},~\eqref{eq:app-Htk-norm}, and~\eqref{eq:app-htK-smooth} into~\eqref{eq:app-meta-smooth-h} gives
\begin{align*}
	&\| \nabla_1 \metaloss_t(\bftheta) - \nabla_1 \metaloss_t(\tilde{\bftheta}) \|_2 \nonumber \\
	&\le \bigg\{ \alpha \sum_{k=0}^{K-1} \big[ L_{H,k} \| \bfg_t^K \|_2 + \gamma^{2K-k-1} G_h G_\loss{} G_{\loss{}h} \big] +  \nonumber \\
	&\quad\;\, \gamma^K G_{\loss{}h} \big\} \| \bftheta - \tilde{\bftheta} \|_2 \nonumber \\
	&= \bigg\{ \alpha \gamma^{K-1} \Big[ K (G_h^2 H_\loss{} + G_\loss{}h_h) + \big( \frac{\gamma^K - 1}{\gamma - 1} - K \big) \times \nonumber \\
	&\quad\;\, (G_h^2 H_\loss{} + G_\loss{}h_h) \| \bfg_t^K \|_2 + \gamma\frac{\gamma^K - 1}{\gamma - 1} G_h G_\loss{} G_{\loss{}h} \Big] + \nonumber \\
	&\quad\;\, \gamma^K G_{\loss{}h} \big\} \| \bftheta - \tilde{\bftheta} \|_2 \nonumber \\
	&= \big( C_{G, 2} \| \bfg_t^K \|_2 + \gamma^{2K} G_{\loss{}h} \big) \| \bftheta - \tilde{\bftheta} \|_2
\end{align*}
where the last equation utilizes the definition of $C_{G, h}$ and that $\gamma = 1+\alpha G_h G_\loss{}$. 

Similarly, taking $\expect_t$ on both side and using Jensen's inequality to exchange $\expect_t$ with $\| \cdot \|_2$ finish the proof of~\eqref{eq:meta-smooth-h}. 

\section{Proof of Proposition~\ref{prop:Ghat-moment}}
First, it directly follows from Jensen's inequality and the unbiasedness of $\hat{G}_{\metaloss,j}^r$ that
\begin{equation*}
	\expect \bigg[ \frac{1}{\hat{G}_{\metaloss,j}^r} \bigg] \ge \frac{1}{\expect \hat{G}_{\metaloss,j}^r} = \frac{1}{G_{\metaloss,j}},\; j=1,2. 
\end{equation*}

Next, we bound the second moment of the smoothness estimator. Applying~\cite[Theorem 1]{bounds-reciprocal-moments} with $k = 2$ and $c = \gamma^{2K} G_{\loss{}h}$ renders
\begin{align}
\label{eq:app-bound-2nd-moment}
	\expect \bigg[ \frac{1}{(\hat{G}_{\metaloss,j}^r)^2} \bigg] 
	&\le \frac{\frac{\sigma_j^2}{(\gamma^{2K} G_{\loss{}h})^2} + \frac{\mu_j^4}{[\sigma_j^2 + \mu_j (\mu_j + \gamma^{2K} G_{\loss{}h})]^2}}{\mu_j^2 + \sigma_j^2} \nonumber \\
	&\overset{(a)}{\le} \frac{\frac{\sigma_j^2}{\gamma^{4K} G_{\loss{}h}^2} + \frac{\mu_j^2}{(\mu_j + \gamma^{2K} G_{\loss{}h})^2}}{\mu_j^2 + \sigma_j^2}
\end{align}
where $\mu_j = C_{G,j} \expect_t \| \bfg_t^K \|_2$ and $\sigma_j^2$ are the mean and variance of $\hat{G}_{\metaloss,j}^r - \gamma^{2K} G_{\loss{}h} = \frac{C_{G,j}}{|\hat{\batch}^r|} \sum_{t \in \hat{\batch}^r} \| \bfg_t^K \|_2$, and $(a)$ is due to $\sigma_j^2 \ge 0$. 

By the definition $G_{\metaloss,j} := C_{G,j} \expect_t \| \bfg_t^K \|_2 + \gamma^{2K} G_{\loss{}h} = \mu_j + \gamma^{2K} G_{\loss{}h}$, it follows from~\eqref{eq:app-bound-2nd-moment} that
\begin{align}
\label{eq:app-Ghat-2ndmom}
	G_{\metaloss,j}^2 \expect \bigg[ \frac{1}{(\hat{G}_{\metaloss,j}^r)^2} \bigg] 
	&\le \frac{\frac{\sigma_j^2 (\mu_j + \gamma^{2K} G_{\loss{}h})^2}{\gamma^{4K} G_{\loss{}h}^2} + \mu_j^2}{\mu_j^2 + \sigma_j^2} \nonumber \\
	&\overset{(a)}{\le} \frac{\frac{2 \sigma_j^2 \mu_j^2}{\gamma^{4K} G_{\loss{}h}^2} + 2 \sigma_j^2 + \mu_j^2}{\mu_j^2 + \sigma_j^2}
\end{align}
where $(a)$ uses the inequality $(a+b)^2 \le 2a^2 + 2b^2$. 

To complete the proof, an upper bound for $\sigma_j^2$ can be established via
\begin{align}
\label{eq:app-sigmaj}
	\sigma_j^2 
	&= \frac{C_{G,j}^2}{|\hat{\batch}^r|^2} \sum_{t \in \hat{\batch}^r} \var_t \| \bfg_t^K \|_2 = \frac{C_{G,j}^2}{|\hat{\batch}^r|} \var_t \| \bfg_t^K \|_2 \nonumber \\
	&= \frac{C_{G,j}^2}{|\hat{\batch}^r|} \Big[ \expect_t \| \bfg_t^K \|_2^2 - \expect_t^2 \| \bfg_t^K \|_2 \Big] \nonumber \\
	&\overset{(a)}{\le} \frac{C_{G,j}^2}{|\hat{\batch}^r|} \Big[ \expect_t \| \bfg_t^K \|_2^2 - \| \expect_t \bfg_t^K \|_2^2 \Big] \nonumber \\
	&\overset{(b)}{=} \frac{C_{G,j}^2}{|\hat{\batch}^r|} \Big[ \expect_t \| \bfg_t^K \|_2^2 - 2 \langle \expect_t \bfg_t^K , \expect_{\tilde{t}} \bfg_{\tilde{t}}^K \rangle + \| \expect_{\tilde{t}} \bfg_{\tilde{t}}^K \|_2^2 \Big] \nonumber \\
	&= \frac{C_{G,j}^2}{|\hat{\batch}^r|} \expect_t \| \bfg_t^K - \expect_{\tilde{t}} \bfg_{\tilde{t}}^K \|_2^2 \nonumber \\
	&\le \frac{C_{G,j}^2}{|\hat{\batch}^r|} \expect_t \Big[ \| \bfg_t^K - \expect_{\tilde{t}} \nabla_1 (\loss{val} \circ \nabla_1 h^*) (\bfz_{\tilde{t}}^K; \bftheta_h) \|_2 + \nonumber \\
	&\qquad \| \expect_{\tilde{t}} [ \nabla_1 (\loss{val} \circ \nabla_1 h^*) (\bfz_{\tilde{t}}^K; \bftheta_h) - \bfg_{\tilde{t}}^K ] \|_2 \Big]^2 \nonumber \\
	&\overset{(c)}{\le} \frac{C_{G,j}^2}{|\hat{\batch}^r|} \Big[ \expect_t^\frac{1}{2} \| \bfg_t^K - \expect_{\tilde{t}} \nabla_1 (\loss{val} \circ \nabla_1 h^*) (\bfz_{\tilde{t}}^K; \bftheta_h) \|_2^2 + \nonumber \\
	&\qquad \expect_t^\frac{1}{2} \| \expect_{\tilde{t}} [ \nabla_1 (\loss{val} \circ \nabla_1 h^*) (\bfz_{\tilde{t}}^K; \bftheta_h) - \bfg_{\tilde{t}}^K ] \|_2^2 \Big]^2
\end{align}
where $(a)$ follows from Jensen's inequality, $(b)$ is because $\expect_t \bfg_t^K = \expect_{\tilde{t}} \bfg_{\tilde{t}}^K$, and $(c)$ utilizes Lemma~\ref{lemma:sum-randvar}. 

On one hand, the first term of~\eqref{eq:app-sigmaj} is bounded via
\begin{align*}
&\expect_t^\frac{1}{2} \| \bfg_t^K - \expect_{\tilde{t}} \nabla_1 (\loss{val} \circ \nabla_1 h^*) (\bfz_{\tilde{t}}^K; \bftheta_h) \|_2^2 \nonumber \\
&\overset{(a)}{\le} \expect_{t,\tilde{t}}^\frac{1}{2} \| \bfg_t^K - \nabla_1 (\loss{val}_t \circ \nabla_1 h^*) (\bfz_{\tilde{t}}^K; \bftheta_h) \|_2^2 + \nonumber \\
&\quad \expect_{t,\tilde{t}}^\frac{1}{2} \| \nabla_1 (\loss{val}_t \circ \nabla_1 h^*) (\bfz_{\tilde{t}}^K; \bftheta_h) - \nabla_1 (\loss{val} \circ \nabla_1 h^*) (\bfz_{\tilde{t}}^K; \bftheta_h) \|_2^2 \nonumber \\
&\overset{(b)}{\le} G_{\loss{}h} \expect_{t,\tilde{t}}^\frac{1}{2} \| \bfz_t^K - \bfz_{\tilde{t}}^K \|_2^2 + \expect_{t,\tilde{t}}^\frac{1}{2} \big[ \| \nabla_1^2 h^* (\bfz_{\tilde{t}}^K; \bftheta_h) \|_2^2 \times \nonumber \\
&\qquad \| \nabla \loss{val}_t (\bfphi_{\tilde{t}}^K) - \nabla \loss{val} (\bfphi_{\tilde{t}}^K) \|_2^2 \big] \\
&\overset{(c)}{\le}  G_{\loss{}h} \zeta \sigma + G_h \sigma
\end{align*}
where $(a)$ follows from Jensen's inequality and Lemma~\ref{lemma:sum-randvar}, $(b)$ is due to Assumption~\ref{as:loss-fn}.iv), $(c)$ leverages Lemma~\ref{lemma:dist-ztk}, and Assumptions~\ref{as:inv-mirror-map}.i) and~\ref{as:stat-property}. 

On the other hand, the second term in~\eqref{eq:app-sigmaj} has bound
\begin{align*}
&\expect_t^\frac{1}{2} \| \expect_{\tilde{t}} [ \nabla_1 (\loss{val} \circ \nabla_1 h^*) (\bfz_{\tilde{t}}^K; \bftheta_h) - \bfg_{\tilde{t}}^K ] \|_2^2 \nonumber \\
&\overset{(a)}\le \expect_{t,\tilde{t}}^\frac{1}{2} \| \nabla_1 (\loss{val} \circ \nabla_1 h^*) (\bfz_{\tilde{t}}^K; \bftheta_h) - \bfg_{\tilde{t}}^K \|_2^2 \nonumber \\
&\le \expect_{t,\tilde{t}}^\frac{1}{2} [\| \nabla_1^2 h^* (\bfz_{\tilde{t}}^K; \bftheta_h) \|_2^2 \| \nabla \loss{val} (\bfphi_{\tilde{t}}^K) - \nabla \loss{val}_{\tilde{t}} (\bfphi_{\tilde{t}}^K) \|_2^2] \nonumber \\
&\overset{(b)}{\le} G_h \sqrt{T} \sigma
\end{align*}
where $(a)$ comes from Jensen's inequality, $(b)$ utilizes Assumption~\ref{as:inv-mirror-map}.i) and Lemma~\ref{lemma:var-grad-phitk}. 

Now, combining these two bounds with~\eqref{eq:app-sigmaj} leads to
\begin{equation}
\label{eq:app-sigmaj-bound}
\sigma_j^2 \le 
\frac{C_{G,j}^2}{|\hat{\batch}^r|} \big[ G_{\loss{}h} \zeta + G_h (1+\sqrt{T}) \big]^2 \sigma^2 \overset{(a)}{\le} \frac{\gamma^{4K} G_{\loss{}h}^2}{2} \sqrt{T}
\end{equation}
where $(a)$ uses that $|\hat{\batch}^r| \ge \frac{2 C_{G,j}^2 [ G_{\loss{}h} \zeta + G_h (1+\sqrt{T})]^2}{\gamma^{4K} G_{\loss{}h}^2 \sqrt{T}} \sigma^2$. 

Next, plugging~\eqref{eq:app-sigmaj-bound} back to~\eqref{eq:app-Ghat-2ndmom} renders
\begin{equation*}
	G_{\metaloss,j}^2 \expect \bigg[ \frac{1}{(\hat{G}_{\metaloss,j}^r)^2} \bigg] \le \frac{(\sqrt{T} + 1) \mu_j^2 + 2\sigma_j^2}{\mu_j^2 + \sigma_j^2} \le \sqrt{T} + 1
\end{equation*}
where the last inequality is because of $T \ge 1$.

\section{Proof of Proposition~\ref{prop:metagrad-2nd-moment}}

On one hand, applying the chain rule~\eqref{eq:app-metagrad-z} yields
\begin{align}
\label{eq:app-gtK-bound-metagrad}
&\expect_t^\frac{1}{2} \| \nabla_1 \metaloss_t (\bftheta) \|_2^2 = \expect_t \bigg\| \prod_{k=0}^{K-1} \big[ \bfI_d - \alpha \bfG_t^k \big] \bfg_t^K \bigg\|_2^2 \nonumber \\
&\overset{(a)}{\le} \gamma^K \expect_t^\frac{1}{2} \| \bfg_t^K \|_2^2 \nonumber \\
&\le \gamma^K \big[ \expect_t^\frac{1}{2} \| \nabla_1 (\loss{val} \circ \nabla_1 h^*) (\bfz_t^K; \bftheta_h) \|_2^2 + \nonumber \\
&\qquad \expect_t^\frac{1}{2} \| \nabla_1 (\loss{val} \circ \nabla_1 h^*) (\bfz_t^K; \bftheta_h) - \bfg_t^K \|_2^2 \big] \nonumber \\
&\overset{(b)}{\le} \gamma^K \big[ \expect_t^\frac{1}{2} \| \nabla_1 (\loss{val} \circ \nabla_1 h^*) (\bfz_t^K; \bftheta_h) \|_2^2 + \nonumber \\
&\qquad \expect_t^\frac{1}{2} ( \| \nabla_1^2 h^* (\bfz_t^K; \bftheta_h) \|_2^2 \| \nabla \loss{val} (\bfphi_t^K) - \nabla \loss{val}_t (\bfphi_t^K) \|_2^2 ) \big] \nonumber \\
&\overset{(c)}{\le} \gamma^K \big[ \expect_t^\frac{1}{2} \| \nabla_1 (\loss{val} \circ \nabla_1 h^*) (\bfz_t^K; \bftheta_h) \|_2^2 + G_h \sqrt{T} \sigma \big] \nonumber \\
&\overset{(d)}{\le} \frac{\gamma^K}{2 - \gamma^K} \Big[ \| \nabla_1 \metaloss (\bftheta) \|_2 + (G_{\loss{}h} \zeta + G_h \sqrt{T}) \sigma \Big]
\end{align}
where $(a)$ leverages~\eqref{eq:app-gamma}, $(b)$ is due to $(a+b)^2 \le 2(a^2+b^2)$, $(c)$ comes from Assumption~\ref{as:inv-mirror-map}.i) and Lemma~\ref{lemma:var-grad-phitk}, and $(d)$ is by utilizing Lemma~\ref{lemma:metaloss-2ndmom-z}. 

On the other hand,~\eqref{eq:app-metagrad-h} suggests
\begin{align}
\label{eq:app-metagrad-h-bound}
&\expect_t^\frac{1}{2} \| \nabla_2 \metaloss_t (\bftheta) \|_2^2 = \expect_t^\frac{1}{2} \bigg\| -\alpha \sum_{k=0}^{K-1} \bfH_t^k \bfg_t^K + \bfh_t^K \bigg\|_2^2 \nonumber \\
&\overset{(a)}{\le} \alpha \sum_{k=0}^K \expect_t^\frac{1}{2} \| \bfH_t^k \bfg_t^K \|_2^2 + \expect_t^\frac{1}{2} \| \bfh_t^K \|_2^2 \nonumber \\
&\overset{(b)}{\le} \alpha \sum_{k=0}^K \gamma^{K-k-1} G_h G_{\loss{}} \expect_t^\frac{1}{2} \| \bfg_t^K \|_2^2 + \expect_t^\frac{1}{2}  \| \bfh_t^K \|_2^2 \nonumber \\
&= (\gamma^K - 1) \expect_t^\frac{1}{2} \| \bfg_t^K \|_2^2 + \expect_t^\frac{1}{2} \| \bfh_t^K \|_2^2
\end{align}
where $(a)$ relies on Lemma~\ref{lemma:sum-randvar}, and $(b)$ is from~\eqref{eq:app-Htk-norm}. 

Notice that
\begin{align}
\label{eq:app-htK-bound-metagrad}
&\expect_t^\frac{1}{2} \| \bfh_t^k \|_2^2 \nonumber\\
&\le \expect_t^\frac{1}{2} \| \nabla_2 (\loss{val} \circ \nabla_1 h^*) (\bfz_t^K; \bftheta_h) \|_2^2 + \nonumber \\
&\qquad \expect_t^\frac{1}{2} \| \nabla_2 (\loss{val} \circ \nabla_1 h^*) (\bfz_t^K; \bftheta_h)- \bfh_t^K \|_2^2 \nonumber \\
&\le \expect_t^\frac{1}{2} \| \nabla_2 (\loss{val} \circ \nabla_1 h^*) (\bfz_t^K; \bftheta_h) \|_2^2 + \nonumber \\
&\qquad \expect_t^\frac{1}{2} \big( \| \nabla_2 \nabla_1 h^* (\bfz_t^K; \bftheta_h) \|_2^2 \| \nabla \loss{val} (\bfphi_t^K) - \nabla \loss{val}_t (\bfphi_t^K) \|_2^2 \big) \nonumber \\
&\overset{(a)}{\le} \expect_t^\frac{1}{2} \| \nabla_2 (\loss{val} \circ \nabla_1 h^*) (\bfz_t^K; \bftheta_h) \|_2^2 + G_h \sqrt{T} \sigma \\
&\overset{(b)}{\le} \frac{(\gamma^K-1) \| \nabla_1 \metaloss(\bftheta) \|_2 + (G_{\loss{}h} \zeta + G_h \sqrt{T}) \sigma}{2-\gamma^K} + \| \nabla_2 \metaloss(\bftheta) \|_2 \nonumber
\end{align}
where $(a)$ is by Assumption~\ref{as:inv-mirror-map}.i) and Lemma~\ref{lemma:var-grad-phitk}, and $(b)$ uses Lemma~\ref{lemma:metaloss-2ndmom-h}. 

Combining~\eqref{eq:app-htK-bound-metagrad} and the upper bound of $\expect_t^\frac{1}{2} \| \bfg_t^k \|_2^2$ in~\eqref{eq:app-gtK-bound-metagrad} with~\eqref{eq:app-metagrad-h-bound} leads to the second inequality in Proposition~\ref{prop:metagrad-2nd-moment}.

\section{Proof of Theorem~\ref{thm:converge}}
Denoting by $G_{\metaloss,j}^r$ the Lipschitz-smoothness defined in Proposition~\ref{prop:meta-smooth} computed at $\bftheta^r$, through Lemma~\ref{lemma:Lip-smooth-2var} we obtain
\begin{align*}
\metaloss (\bftheta^{r+1}) 
&\le \metaloss (\bftheta^r) - (\bftheta^{r+1} - \bftheta^r)^\top \nabla \metaloss (\bftheta^r) + \\
&\qquad\frac{G_{\metaloss,1}^r}{2} \| \bftheta_z^{r+1} - \bftheta_z^r \|_2^2 + \frac{G_{\metaloss,2}^r}{2} \| \bftheta_h^{r+1} - \bftheta_h^r \|_2^2 \\
&= \metaloss (\bftheta^r) - \frac{\beta_1^r}{B} \Big[ \sum_{t \in \batch^r} \nabla_1 \metaloss_t (\bftheta^r) \Big]^\top \nabla_1 \metaloss (\bftheta^r) - \\
&\qquad \frac{\beta_2^r}{B} \Big[ \sum_{t \in \batch^r} \nabla_1 \metaloss_t (\bftheta^r) \Big]^\top \nabla_2 \metaloss (\bftheta^r) + \\
&\qquad \frac{G_{\metaloss,1}^r (\beta_1^r)^2}{2 B^2} \Big\| \sum_{t \in \batch^r} \nabla_1 \metaloss_t (\bftheta^r) \Big\|_2^2 + \\
&\qquad \frac{G_{\metaloss,2}^r (\beta_2^r)^2}{2 B^2} \Big\| \sum_{t \in \batch^r} \nabla_2 \metaloss_t (\bftheta^r) \Big\|_2^2. 
\end{align*}

Taking conditional expectation on both sides leads to
\begin{align}
\label{eq:app-condexp}
&\expect_{\batch^r, \hat{\batch}^r} [\metaloss (\bftheta^{r+1}) | \bftheta^r] \nonumber \\
&\overset{(a)}{\le} \metaloss (\bftheta^r) - \frac{1}{C_\beta} \expect_{\hat{\batch}^r} \Big[ \frac{1}{\hat{G}_{\metaloss,1}} \Big| \bftheta^r \Big] \| \nabla_1 \metaloss (\bftheta^r) \|_2^2 - \nonumber \\
&\qquad \frac{1}{C_\beta} \expect_{\hat{\batch}^r} \Big[ \frac{1}{\hat{G}_{\metaloss,2}} \Big| \bftheta^r \Big] \| \nabla_2 \metaloss (\bftheta^r) \|_2^2 + \nonumber \\
&\qquad \frac{G_{\metaloss,1}^r}{2 C_\beta^2} \expect_{\hat{\batch}^r} \Big[ \frac{1}{\hat{G}_{\metaloss,1}^2} \Big| \bftheta^r \Big] \expect_{\batch^r} \Big\| \frac{1}{B} \sum_{t \in \batch^r} \nabla_1 \metaloss_t (\bftheta^r) \Big\|_2^2 + \nonumber \\
&\qquad \frac{G_{\metaloss,2}^r}{2 C_\beta^2} \expect_{\hat{\batch}^r} \Big[ \frac{1}{\hat{G}_{\metaloss,2}^2} \Big| \bftheta^r \Big] \expect_{\batch^r} \Big\| \frac{1}{B} \sum_{t \in \batch^r} \nabla_2 \metaloss_t (\bftheta^r) \Big\|_2^2
\end{align}
where $(a)$ relies on the definition of $\beta_j^r$, and that $\batch^r$ is independent of $\hat{\batch}^r$. 

Note that
\begin{align}
\label{eq:app-condexp-z}
&\expect_{\batch^r} \Big\| \frac{1}{B} \sum_{t \in \batch^r} \nabla_1 \metaloss_t (\bftheta^r) \Big\|_2^2 \nonumber \\
&= \frac{1}{B^2} \sum_{t \in \batch^r} \expect_t \| \nabla_1 \metaloss_t (\bftheta^r) \|_2^2 + \nonumber \\
&\qquad \frac{1}{B^2} \sum_{\substack{t, \tilde{t} \in \batch^r\\ t \ne \tilde{t}}} \expect_t \nabla_1 \metaloss_t (\bftheta^r)^\top \expect_{\tilde{t}} \nabla_1 \metaloss_{\tilde{t}} (\bftheta^r) \nonumber \\
&\overset{(a)}{\le} \frac{C_{\metaloss,1}^2}{B} \big( \| \nabla_1 \metaloss (\bftheta^r) \|_2 + C_{\metaloss,2} \sigma \big)^2 + \frac{B-1}{B} \| \nabla_1 \metaloss (\bftheta^r) \|_2^2 \nonumber \\
&\overset{(b)}{\le} \frac{1}{B} \big[ (2C_{\metaloss,1}^2 + B - 1) \| \nabla_1 \metaloss (\bftheta^r) \|_2^2 + 2C_{\metaloss,1}^2 C_{\metaloss,2}^2 \sigma^2 \big]
\end{align}
where $(a)$ follows from~\eqref{eq:metagrad-2nd-moment-z} of Proposition~\ref{prop:metagrad-2nd-moment}, and $(b)$ is because $(a+b)^2 \le 2(a^2 + b^2)$. 

Likewise, utilizing~\eqref{eq:metagrad-2nd-moment-h} in Proposition~\ref{prop:metagrad-2nd-moment} gives
\begin{align}
\label{eq:app-condexp-h}
&\expect_{\batch^r} \Big\| \frac{1}{B} \sum_{t \in \batch^r} \nabla_2 \metaloss_t (\bftheta^r) \Big\|_2^2 \nonumber \\
&\le \frac{1}{B} \Big[ (C_{\metaloss,1} - 1) \| \nabla_1 \metaloss (\bftheta) \|_2 + C_{\metaloss,1} C_{\metaloss,2} \sigma + \| \nabla_2 \metaloss (\bftheta) \|_2 \Big]^2 + \nonumber \\
&\qquad \frac{B-1}{B} \| \nabla_2 \metaloss (\bftheta^r) \|_2^2 \nonumber \\
&\overset{(a)}{\le} \frac{1}{B} \big[ 3(C_{\metaloss,1} - 1)^2 \| \nabla_1 \metaloss (\bftheta) \|_2^2 + (B+2) \| \nabla_2 \metaloss (\bftheta) \|_2^2 + \nonumber \\
&\qquad 3 C_{\metaloss,1}^2 C_{\metaloss,2}^2 \sigma^2 \big]
\end{align}
where $(a)$ is by leveraging $(a+b+c)^2 \le 3(a^2 + b^2 + c^2)$. 

Plug Proposition~\ref{prop:Ghat-moment},~\eqref{eq:app-condexp-z}, and~\eqref{eq:app-condexp-h} into~\eqref{eq:app-condexp} to arrive at
\begin{align}
\label{eq:app-condexp-upper}
&\expect_{\batch^r, \hat{\batch}^r} [\metaloss (\bftheta^{r+1}) | \bftheta^r] 
\le \metaloss (\bftheta^r) - \frac{1}{C_\beta G_{\metaloss,1}^r} \bigg[ 1 - \frac{\sqrt{T} + 1}{2 C_\beta B} \times \nonumber \\
&\qquad \Big( 2 C_{\metaloss,1}^2 + B - 1 + \frac{3 (C_{\metaloss,1} - 1)^2 G_{\metaloss,1}^r}{G_{\metaloss,2}^r} \Big) \bigg] \| \nabla_1 \metaloss (\bftheta^r) \|_2^2 - \nonumber \\
&\qquad \frac{1}{C_\beta G_{\metaloss,2}^r} \bigg[ 1 - \frac{(\sqrt{T} + 1)(B+2)}{2 C_\beta B} \bigg] \| \nabla_2 \metaloss (\bftheta^r) \|_2^2 + \nonumber \\
&\qquad \frac{C_{\metaloss,1}^2 C_{\metaloss,2}^2 (\sqrt{T} + 1)}{2 C_\beta^2 B} \Big( \frac{2}{G_{\metaloss,1}^r} + \frac{3}{G_{\metaloss,2}^r} \Big) \sigma^2. 
\end{align}
From the definition $G_{\metaloss,j}^r := C_{G,j} \expect_t \| \bfg_t^K \|_2 + \gamma^{2K} G_{\loss{}h}$, it holds that $G_{\metaloss,j}^r \ge \gamma^{2K} G_{\loss{}h}$, and $\frac{G_{\metaloss,1}^r}{G_{\metaloss,2}^r} \le \max\{ \frac{C_{G,1}^r}{C_{G,2}^r}, 1 \}$. Thus, it can be verified using the condition on $B$ of Theorem~\ref{thm:converge} that, the coefficients of $\| \nabla_j \metaloss (\bftheta^r) \|_2^2,\,j=1,2$ in~\eqref{eq:app-condexp-upper} are both greater than $0$. 
Again using the definition of $G_{\metaloss,j}^r$, we can also upper bound it via
\begin{equation*}
G_{\metaloss,j}^r \overset{(a)}{\le} \frac{C_{G,j}}{2 - \gamma^K} \big( \| \nabla_1 \metaloss (\bftheta^r) \|_2 + C_{\metaloss,2} \big) + \gamma^{2K} G_{\loss{}h}
\end{equation*}
where $(a)$ follows from~\eqref{eq:app-gtK-bound-metagrad} and the inequality $\expect_t \| \bfg_t^K \|_2 \le \expect_t^\frac{1}{2} \| \bfg_t^K \|_2^2$. Next, applying the upper and lower bounds of $G_{\metaloss,j}^r$ and the lower bound of $\frac{G_{\metaloss,1}^r}{G_{\metaloss,2}^r}$ to~\eqref{eq:app-condexp-upper} render
\begin{align}
\label{eq:app-def-eta}
&\expect_{\batch^r, \hat{\batch}^r} [\metaloss (\bftheta^{r+1}) | \bftheta^r]  \nonumber \\
&\le \metaloss (\bftheta^r) - \frac{ \frac{2-\gamma^K}{C_\beta C_{G,1}} (1 - \frac{C_{B,2} + B}{C_{B,1} B}) \| \nabla_1 \metaloss (\bftheta^r) \|_2^2 }{ \| \nabla_1 \metaloss (\bftheta^r) \|_2 + C_{\metaloss,2} + \gamma^{2K} (2 - \gamma^K) \frac{G_{\loss{}h}}{C_{G,1}} } - \nonumber \\
&\qquad \frac{\frac{2 - \gamma^K}{C_\beta C_{G,2}}(1 - \frac{B+2}{C_{B,1} B}) \| \nabla_2 \metaloss (\bftheta^r) \|_2^2}{\| \nabla_1 \metaloss (\bftheta^r) \|_2 + C_{\metaloss,2} + \gamma^{2K} (2 - \gamma^K) \frac{G_{\loss{}h}}{C_{G,2}} } + \nonumber \\
&\qquad \frac{5C_{\metaloss,1}^2 C_{\metaloss,2}^2}{\gamma^{2K} G_{\loss{}h} C_{B,1} C_\beta B} \sigma^2 \\
&:= \metaloss (\bftheta^r) - \frac{\eta_1 \| \nabla_1 \metaloss (\bftheta^r) \|_2^2}{\| \nabla_1 \metaloss (\bftheta^r) \|_2 + \eta_2} - \frac{\eta_3 \| \nabla_2 \metaloss (\bftheta^r) \|_2^2}{\| \nabla_1 \metaloss (\bftheta^r) \|_2 + \eta_4} + \frac{\eta_5}{B}. \nonumber
\end{align}

Taking expectation wrt $\bftheta^r$ leads to
\begin{align*}
&\expect \frac{\eta_1 \| \nabla_1 \metaloss (\bftheta^r) \|_2^2}{\| \nabla_1 \metaloss (\bftheta^r) \|_2 + \eta_2} + \expect \frac{\eta_3 \| \nabla_2 \metaloss (\bftheta^r) \|_2^2}{\| \nabla_1 \metaloss (\bftheta^r) \|_2 + \eta_4} \le \nonumber \\
&\hspace{5cm} \expect [\metaloss(\bftheta^r) - \metaloss(\bftheta^{r+1})] + \frac{\eta_5}{B}.
\end{align*}

Then, averaging this inequality from $r=0$ to $r=R-1$, and defining discrete random variable $\rho$ which distrbutes uniformly over $\{ 0,\ldots,R-1 \}$, we obtain
\begin{align*}
&\expect \frac{\eta_1 \| \nabla_1 \metaloss (\bftheta^\rho) \|_2^2}{\| \nabla_1 \metaloss (\bftheta^\rho) \|_2 + \eta_2} + \expect \frac{\eta_3 \| \nabla_2 \metaloss (\bftheta^\rho) \|_2^2}{\| \nabla_1 \metaloss (\bftheta^\rho) \|_2 + \eta_4} \nonumber \\
&\le \frac{\expect [\metaloss(\bftheta^0) - \metaloss(\bftheta^R)]}{R} + \frac{\eta_5}{B} \nonumber \\
&\le \frac{\expect [\metaloss(\bftheta^0) - \inf_{\bftheta} \metaloss(\bftheta)]}{R} + \frac{\eta_5}{B} := \frac{\Delta}{R} + \frac{\eta_5}{B}. 
\end{align*}
As the two terms on the left-hand side of the inequality above are both non-negative, it implies that
\begin{subequations}
\begin{align}
\label{eq:app-quadeq-metagrad-z}
\expect \frac{\eta_1 \| \nabla_1 \metaloss (\bftheta^\rho) \|_2^2}{\| \nabla_1 \metaloss (\bftheta^\rho) \|_2 + \eta_2} 
&\le \frac{\Delta}{R} + \frac{\eta_5}{B}, \\
\label{eq:app-quadeq-metagrad-h}
\expect \frac{\eta_3 \| \nabla_2 \metaloss (\bftheta^\rho) \|_2^2}{\| \nabla_1 \metaloss (\bftheta^\rho) \|_2 + \eta_4}
&\le \frac{\Delta}{R} + \frac{\eta_5}{B}.
\end{align}
\end{subequations}

Notice that $\eta_2 > 0$ by its definition, and thereby it can be verified that $\frac{\eta_1 \| \nabla_1 \metaloss (\bftheta^\rho) \|_2^2}{\| \nabla_1 \metaloss (\bftheta^\rho) \|_2 + \eta_2}$ is a convex function of $\| \nabla_1 \metaloss (\bftheta^\rho) \|_2 \ge 0$. Using Jensen's inequality on~\eqref{eq:app-quadeq-metagrad-z} gives
\begin{equation*}
\frac{\eta_1 \expect^2 \| \nabla_1 \metaloss (\bftheta^\rho) \|_2}{\expect \| \nabla_1 \metaloss (\bftheta^\rho) \|_2 + \eta_2} 
\le \frac{\Delta}{R} + \frac{\eta_5}{B}
\end{equation*}
which can be solved to yield the solution
\begin{align*}
\expect \| \nabla_1 \metaloss (\bftheta^\rho) \|_2 
&\le \frac{1}{2\eta_1} \Big( \frac{\Delta}{R} + \frac{\eta_5}{B} \Big) + \nonumber \\
&\qquad \sqrt{ \frac{1}{4\eta_1^2} \Big(\frac{\Delta}{R} + \frac{\eta_5}{B} \Big)^2 + \frac{\eta_2}{\eta_1} \Big( \frac{\Delta}{R} + \frac{\eta_5}{B} \Big)}. 
\end{align*}

Moreover, consider function $f(x_1, x_2) := \frac{x_1^2}{x_2+c}$ where $x_1,x_2 \ge 0$ and $c > 0$. It is easy to observe that
\begin{align*}
	\nabla^2 f(x_1, x_2) &= \frac{2}{x_2+c} 
	\left[\begin{matrix} 
		1 & -\frac{x_1}{x_2+c} \\ 
		-\frac{x_1}{x_2+c} & \frac{x_1^2}{(x_2 + c)^2} 
	\end{matrix}\right] \\
	&= \frac{2}{x_2+c} 
	\left[\begin{matrix}
		1 & -\frac{x_1}{x_2+c}
	\end{matrix} \right] 
	\left[\begin{matrix}
		1 \\ -\frac{x_1}{x_2+c}
	\end{matrix}\right]
\end{align*}
which has eigenvalues $\frac{2}{x_2+c} [1 + \frac{x_1^2}{(x_2+c)^2}] > 0$ and $0$. As a result, $f$ is convex and thus applying Jensen's inequality on~\eqref{eq:app-quadeq-metagrad-h} results in
\begin{equation*}
\frac{\eta_3 \expect^2 \| \nabla_2 \metaloss (\bftheta^\rho) \|_2}{\expect \| \nabla_1 \metaloss (\bftheta^\rho) \|_2 + \eta_4} 
\le \frac{\Delta}{R} + \frac{\eta_5}{B}.
\end{equation*}
This inequality suggests
\begin{equation*}
\expect \| \nabla_2 \metaloss (\bftheta^\rho) \|_2 \le \sqrt{\frac{1}{\eta_3} \big( \expect \| \nabla_1 \metaloss (\bftheta^\rho) \|_2 + \eta_4 \big) \Big( \frac{\Delta}{R} + \frac{\eta_5}{B} \Big)}.
\end{equation*}
The proof is thus completed. 

\section{Useful lemmas}
This section provides several critical lemmas for our proof. 

\begin{lemma}
\label{lemma:sum-randvar}
For random variables $\{ X_n \in \real \}_{n=1}^N$, it holds that
\begin{equation*}
\expect \bigg[ \sum_{n=1}^N X_n \bigg]^2 \le \bigg[ \sum_{n=1}^N \expect^\frac{1}{2} X_n^2 \bigg]^2. 
\end{equation*}

\begin{proof}
It follows that
\begin{align*}
\expect \bigg[ \sum_{n=1}^N X_n \bigg]^2 
&= \sum_{n=1}^N \expect X_n^2 + 2\sum_{1 \le m < n \le N} \expect X_m X_n \\
&\overset{(a)}{\le} \sum_{n=1}^N \expect X_n^2 + 2\sum_{1 \le m < n \le N} \expect^\frac{1}{2} X_m^2 \expect^\frac{1}{2} X_n^2 \\
&= \bigg[ \sum_{n=1}^N \expect^\frac{1}{2} X_n^2 \bigg]^2
\end{align*}
where $(a)$ is because $2 X_m X_n \le C X_m^2 + \frac{1}{C} X_n^2$ with $C = \expect^\frac{1}{2} X_n^2 / \expect^\frac{1}{2} X_m^2 > 0$. 
\end{proof}
\end{lemma}

\begin{lemma}
\label{lemma:ztk-Lip}
With Assumptions~\ref{as:loss-fn}.ii) and~\ref{as:inv-mirror-map}.i) in effect, it holds for $ t = 1,\ldots,T$, $k=0,\ldots,K$, and $\forall \bftheta, \tilde{\bftheta} \in \real^D$ that
\begin{equation*}
	\| \bfz_t^k - \tilde{\bfz}_t^k \|_2 \le \gamma^k \| \bftheta_z - \tilde{\bftheta}_z \|_2 + (\gamma^k - 1) \| \bftheta_h - \tilde{\bftheta}_h \|_2. 
\end{equation*}
\end{lemma}

\begin{proof}
By~\eqref{eq:MetaMDA-obj-task} we obtain
\begin{align*}
\| \bfz_t^k - \tilde{\bfz}_t^k \|_2 
&= \| \bfz_t^{k-1} - \alpha \nabla \loss{trn}_t(\nabla_1 h^* (\bfz_t^{k-1}; \bftheta_h)) + \nonumber \\
&\qquad \tilde{\bfz}_t^{k-1} - \alpha \nabla \loss{trn}_t(\nabla_1 h^* (\tilde{\bfz}_t^{k-1}; \tilde{\bftheta}_h)) \|_2 \nonumber \\
&\le \| \bfz_t^{k-1} - \tilde{\bfz}_t^{k-1} \|_2 {+} \alpha \| \nabla \loss{trn}_t(\nabla_1 h^* (\bfz_t^{k-1}; \bftheta_h)) - \nonumber \\
&\qquad \nabla \loss{trn}_t(\nabla_1 h^* (\tilde{\bfz}_t^{k-1}; \tilde{\bftheta}_h)) \|_2 \nonumber \\
&\overset{(a)}{\le} \| \bfz_t^{k-1} - \tilde{\bfz}_t^{k-1} \|_2 +  \nonumber \\
&\quad\;\; \alpha G_\loss{} \| \nabla_1 h^* (\bfz_t^{k-1}; \bftheta_h) -\nabla_1 h^* (\tilde{\bfz}_t^{k-1}; \tilde{\bftheta}_h) \|_2 \nonumber \\
&\overset{(b)}{\le} \gamma \| \bfz_t^{k-1} - \tilde{\bfz}_t^{k-1} \|_2 + (\gamma - 1) \| \bftheta_h - \tilde{\bftheta}_h \|_2 \nonumber \\
&\overset{(c)}{\le} \gamma^k \| \bfz_t^0 - \tilde{\bfz}_t^0 \|_2 + (\gamma - 1) \sum_{l=0}^{k-1} \gamma^l \| \bftheta_h - \tilde{\bftheta}_h \|_2 \nonumber \\
&= \gamma^k \| \bftheta_z - \tilde{\bftheta}_z \|_2 + (\gamma^k - 1) \| \bftheta_h - \tilde{\bftheta}_h \|_2
\end{align*}
where $(a)$ and $(b)$ respectively rely on Assumptions~\ref{as:loss-fn}.ii) and~\ref{as:inv-mirror-map}.i), and $(c)$ telescopes the inequality. 
\end{proof}

\begin{lemma}
\label{lemma:var-grad-phitk}
With Assumption~\ref{as:stat-property} in effect, it holds for $t = 1,\ldots,T$, $k = 1,\ldots,K$, and $\forall \bftheta \in \real^D$ that
\begin{equation*}
	\expect_t^\frac{1}{2} \| \nabla \loss{val}_t (\bfphi_t^k) - \nabla \loss{val} (\bfphi_t^k) \|_2^2 \le  \sqrt{T} \sigma. \nonumber \\
\end{equation*}
\end{lemma}

\begin{proof}
From Assumption~\ref{as:stat-property} we have
\begin{align*}
&\expect_{\tilde{t}} \| \nabla \loss{val}_{\tilde{t}} (\bfphi_t^k) - \nabla \loss{val} (\bfphi_t^k) \|_2^2 \nonumber \\
&= \frac{1}{T} \sum_{\tilde{t}=1}^T \| \nabla \loss{val}_{\tilde{t}} (\bfphi_t^k) - \nabla \loss{val} (\bfphi_t^k) \|_2^2 \le \sigma^2
\end{align*}
which suggests 
\begin{equation}
\label{eq:app-var-grad-phitk-bound}
\| \nabla \loss{val}_{\tilde{t}} (\bfphi_t^k) - \nabla \loss{val} (\bfphi_t^k) \|_2^2 \le T \sigma^2,\; \forall \tilde{t}. 
\end{equation}
Assigning $\tilde{t} = t$ and taking $\expect_t^\frac{1}{2}$ on both sides lead to Lemma~\ref{lemma:var-grad-phitk}. 
%
\end{proof}

\begin{lemma}
\label{lemma:dist-ztk}
With Assumptions~\ref{as:loss-fn}.ii), \ref{as:stat-property} and, \ref{as:inv-mirror-map}.i) in effect, it holds for $t = 1,\ldots,T$, and $\forall \bftheta \in \real^D$ that
\begin{equation*}
	\expect_{t,\tilde{t}}^\frac{1}{2} \| \bfz_t^K - \bfz_{\tilde{t}}^K \|_2^2 \le \zeta \sigma
\end{equation*}
where $\zeta := 2\alpha + \frac{(\gamma^K - \gamma)(\sqrt{T} + 1)}{G_h G_{\loss{}}}$. 
\end{lemma}

\begin{proof}
First, notice that
\begin{align}
\label{eq:app-dist-ztk}
&\| \bfz_t^K - \bfz_{\tilde{t}}^K \|_2  \nonumber \\
&\le \| \bfz_t^{K-1} - \bfz_{\tilde{t}}^{K-1} \|_2 + \alpha \| \nabla \loss{trn}_t (\bfphi_t^{K-1}) - \nabla \loss{trn}_{\tilde{t}} (\bfphi_{\tilde{t}}^{K-1}) \|_2 \nonumber \\
&\le \| \bfz_t^{K-1} - \bfz_{\tilde{t}}^{K-1} \|_2 + \alpha \| \nabla \loss{trn}_t (\bfphi_t^{K-1}) - \nabla \loss{trn}_t (\bfphi_{\tilde{t}}^{K-1}) \|_2 + \nonumber \\
&\qquad \alpha \| \nabla \loss{trn}_t (\bfphi_{\tilde{t}}^{K-1}) - \nabla \loss{trn}_{\tilde{t}} (\bfphi_{\tilde{t}}^{K-1}) \|_2 \nonumber \\
&\overset{(a)}{\le} \gamma \| \bfz_t^{K-1} - \bfz_{\tilde{t}}^{K-1} \|_2 + \alpha ( \| \nabla \loss{trn}_t (\bfphi_{\tilde{t}}^{K-1}) -  \nonumber \\
&\qquad \nabla \loss{trn} (\bfphi_{\tilde{t}}^{K-1}) \|_2 + \| \nabla \loss{trn} (\bfphi_{\tilde{t}}^{K-1}) - \nabla \loss{trn}_{\tilde{t}} (\bfphi_{\tilde{t}}^{K-1}) \|_2 ) \nonumber \\
&\overset{(b)}{\le} \alpha \sum_{k=0}^{K-1} \gamma^k ( \| \nabla \loss{trn}_t (\bfphi_{\tilde{t}}^k) - \nabla \loss{trn} (\bfphi_{\tilde{t}}^k) \|_2 + \nonumber \\
&\qquad \| \nabla \loss{trn}_{\tilde{t}} (\bfphi_{\tilde{t}}^k) - \nabla \loss{trn} (\bfphi_{\tilde{t}}^k) \|_2 )
\end{align}
where $(a)$ comes from Assumptions~\ref{as:loss-fn}.ii) and \ref{as:inv-mirror-map}.i), and $(b)$ leverages telescoping. 
Next, using Lemma~\ref{lemma:sum-randvar} leads to
\begin{align*}
\expect_{t,\tilde{t}}^\frac{1}{2} \| \bfz_t^K - \bfz_{\tilde{t}}^K \|_2^2 
&\le \alpha \Big[ \sum_{k=0}^{K-1} \gamma^{k} \expect_{t,\tilde{t}}^\frac{1}{2} \big( \| \nabla \loss{trn}_t (\bfphi_{\tilde{t}}^k) {-} \nabla \loss{trn} (\bfphi_{\tilde{t}}^k) \|_2 {+} \nonumber \\
&\qquad \| \nabla \loss{trn}_{\tilde{t}} (\bfphi_{\tilde{t}}^k) - \nabla \loss{trn} (\bfphi_{\tilde{t}}^k) \|_2 \big)^2 \Big] \\
&\le \alpha \Big[ \sum_{k=0}^{K-1} \gamma^{k} \big( \expect_{t,\tilde{t}}^\frac{1}{2} \| \nabla \loss{trn}_t (\bfphi_{\tilde{t}}^k) {-} \nabla \loss{trn} (\bfphi_{\tilde{t}}^k) \|_2^2 {+} \\
&\qquad \expect_{t,\tilde{t}}^\frac{1}{2} \| \nabla \loss{trn}_{\tilde{t}} (\bfphi_{\tilde{t}}^k) - \nabla \loss{trn} (\bfphi_{\tilde{t}}^k) \|_2^2 \big) \Big] \\
&\le \alpha \bigg[2\sigma + \sum_{k=1}^{K-1} \gamma^k (1 + \sqrt{T}) \sigma \bigg] = \zeta \sigma
\end{align*}
where the last inequality leverages Assumption~\ref{as:stat-property} and Lemma~\ref{lemma:var-grad-phitk}. 
\end{proof}

\begin{lemma}
\label{lemma:metaloss-2ndmom-z}
Suppose Assumptions~\ref{as:loss-fn}-\ref{as:inv-mirror-map} hold. If $\alpha < \frac{2^{1/K} - 1}{G_h G_\loss{}}$, then it holds for $t = 1,\ldots,T$, and $\forall \bftheta \in \real^D$ that
\begin{align*}
\expect_t^\frac{1}{2} \| \nabla_1 (\loss{val} \circ \nabla_1 h^*)(\bfz_t^K; \bftheta_h) \|_2^2 &\le \frac{\| \nabla_1 \metaloss (\bftheta) \|_2 + C_{\loss{}h} \sigma}{2 - \gamma^K}
\end{align*}
where $C_{\loss{}h} := G_{\loss{}h} \zeta  + (\gamma^K - 1) G_h \sqrt{T}$. 
\end{lemma}

\begin{proof}
Using Jensen's inequality, it follows that
\begin{align}
\label{eq:app-metaloss-2ndmom-z}
&\expect_t^\frac{1}{2} \| \nabla_1 (\loss{val} \circ \nabla_1 h^*)(\bfz_t^K; \bftheta_h) \|_2^2 \nonumber \\
&= \expect_t^\frac{1}{2} \| \expect_{\tilde{t}} [\nabla_1 (\loss{val}_{\tilde{t}} {\circ} \nabla_1 h^*)(\bfz_t^K; \bftheta_h) - \nabla_1 \metaloss_{\tilde{t}} (\bftheta)] + \nabla_1 \metaloss (\bftheta) \|_2^2 \nonumber \\
&\le \expect_{t, \tilde{t}}^\frac{1}{2} \big[ \| \nabla_1 (\loss{val}_{\tilde{t}} \circ \nabla_1 h^*)(\bfz_t^K; \bftheta_h) - \nabla_1 \metaloss_{\tilde{t}} (\bftheta) \|_2 + \nonumber \\
&\qquad \| \nabla_1 \metaloss (\bftheta) \|_2 \big]^2 \nonumber \\
&\le \expect_{t, \tilde{t}}^\frac{1}{2} \big[ \| \nabla_1 (\loss{val}_{\tilde{t}} \circ \nabla_1 h^*)(\bfz_t^K; \bftheta_h) - \bfg_{\tilde{t}}^K \|_2 + \nonumber \\
&\qquad \expect_{\tilde{t}} \| \bfg_{\tilde{t}}^K - \nabla_1 \metaloss_{\tilde{t}} (\bftheta) \|_2 + \| \nabla_1 \metaloss (\bftheta) \|_2 \big]^2 \nonumber \\
&\overset{(a)}{\le} \expect_{t, \tilde{t}}^\frac{1}{2} \| \nabla_1 (\loss{val}_{\tilde{t}} \circ \nabla_1 h^*)(\bfz_t^K; \bftheta_h) - \bfg_{\tilde{t}}^K \|_2^2 + \nonumber \\
&\qquad \expect_{\tilde{t}}^\frac{1}{2} \| \bfg_{\tilde{t}}^K - \nabla_1 \metaloss_{\tilde{t}} (\bftheta) \|_2^2 + \| \nabla_1 \metaloss (\bftheta) \|_2
\end{align}
where $(a)$ is due to Lemma~\ref{lemma:sum-randvar}. 

Relying on Assumption~\ref{as:loss-fn}.iv) and Lemma~\ref{lemma:dist-ztk}, the first term in~\eqref{eq:app-metaloss-2ndmom-z} is upper bounded by
\begin{align}
\label{eq:app-metaloss-2ndmom-z-term1}
&\expect_{t, \tilde{t}}^\frac{1}{2} \| \nabla_1 (\loss{val}_{\tilde{t}} \circ \nabla_1 h^*)(\bfz_t^K; \bftheta_h) - \bfg_{\tilde{t}}^K \|_2^2 \nonumber \\
&\le G_{\loss{}h} \expect_{t, \tilde{t}}^\frac{1}{2} \| \bfz_t^K - \bfz_{\tilde{t}}^K \|_2^2 
\le G_{\loss{}h} \zeta \sigma. 
\end{align}
And the second term has upper bound
\begin{align}
\label{eq:app-metaloss-2ndmom-z-term2}
&\expect_{\tilde{t}}^\frac{1}{2} \| \bfg_{\tilde{t}}^K - \nabla_1 \metaloss_{\tilde{t}} (\bftheta) \|_2^2 = \expect_t^\frac{1}{2} \| \bfg_t^K - \nabla_1 \metaloss_t (\bftheta) \|_2^2 \nonumber \\
&\le \expect_t^\frac{1}{2} \Big[ \Big\| \bfI_d - \prod_{k=0}^{K-1} \big[ \bfI_d - \alpha \bfG_t^k \big] \Big\|_2^2 \| \bfg_t^K \|_2^2 \Big] \nonumber \\
&\overset{(a)}{\le} (\gamma^K - 1) \expect_t^\frac{1}{2} \| \bfg_t^K \|_2^2 \nonumber \\
&\overset{(b)}{\le} (\gamma^K - 1) \big[ \expect_t^\frac{1}{2} \| \nabla_1 (\loss{val} \circ \nabla_1 h^*)(\bfz_t^K; \bftheta_h) \|_2^2 + \nonumber \\
&\qquad \expect_t^\frac{1}{2} \| \nabla_1 (\loss{val} \circ \nabla_1 h^*)(\bfz_t^K; \bftheta_h) - \bfg_t^K \|_2^2 \big] \nonumber \\
&\le (\gamma^K - 1) \big[ \expect_t^\frac{1}{2} \| \nabla_1 (\loss{val} \circ \nabla_1 h^*)(\bfz_t^K; \bftheta_h) \|_2^2 +  \\
&\qquad \expect_t^\frac{1}{2} ( \| \nabla_1^2 h^*(\bfz_t^K; \bftheta_h) \|_2^2 \| \nabla \loss{val} (\bfphi_t^K) - \nabla \loss{val}_t (\bfphi_t^K) \|_2^2) \big] \nonumber \\
&\overset{(c)}{\le} (\gamma^K - 1) \big[ \expect_t^\frac{1}{2} \| \nabla_1 (\loss{val} \circ \nabla_1 h^*)(\bfz_t^K; \bftheta_h) \|_2 + G_h \sqrt{T} \sigma \big] \nonumber
\end{align}
where $(a)$ follows from~\cite[Lemma 13]{converge-multistep-MAML}, $(b)$ utilizes Lemma~\ref{lemma:sum-randvar}, and $(c)$ is via Assumption~\ref{as:inv-mirror-map}.i) and Lemma~\ref{lemma:var-grad-phitk}. 

Relating~\eqref{eq:app-metaloss-2ndmom-z} to~\eqref{eq:app-metaloss-2ndmom-z-term1} and~\eqref{eq:app-metaloss-2ndmom-z-term2}, and rearranging the terms yield Lemma~\ref{lemma:metaloss-2ndmom-z}. 
\end{proof}

\begin{lemma}
\label{lemma:metaloss-2ndmom-h}
Suppose Assumptions~\ref{as:loss-fn}-\ref{as:inv-mirror-map} hold. If $\alpha < \frac{2^{1/K} - 1}{G_h G_\loss{}}$, then it holds for $t = 1,\ldots,T$, and $\forall \bftheta \in \real^D$ that
\begin{align*}
&\expect_t^\frac{1}{2} \| \nabla_2 (\loss{val} \circ \nabla_1 h^*)(\bfz_t^K; \bftheta_h) \|_2^2 \le \\
&\qquad \frac{(\gamma^K-1) \| \nabla_1 \metaloss(\bftheta) \|_2 + C_{\loss{}h} \sigma}{2-\gamma^K} + \| \nabla_2 \metaloss(\bftheta) \|_2. 
\end{align*}
\end{lemma}

\begin{proof}
By Jensen's inequality and Lemma~\ref{lemma:sum-randvar}, it follows that
\begin{align}
\label{eq:app-metaloss-1stmom-h}
&\expect_t^\frac{1}{2} \| \nabla_2 (\loss{val} \circ \nabla_1 h^*)(\bfz_t^K; \bftheta_h) \|_2^2 \nonumber \\
&= \expect_t^\frac{1}{2} \| \expect_{\tilde{t}} [\nabla_2 (\loss{val}_{\tilde{t}} \circ \nabla_1 h^*)(\bfz_t^K; \bftheta_h) - \nabla_2 \metaloss_{\tilde{t}} (\bftheta)] + \nabla_2 \metaloss (\bftheta) \|_2^2 \nonumber \\
&\le \expect_{t, \tilde{t}}^\frac{1}{2} \| \nabla_2 (\loss{val}_{\tilde{t}} \circ \nabla_1 h^*)(\bfz_t^K; \bftheta_h) - \nabla_2 \metaloss_{\tilde{t}} (\bftheta) \|_2^2 + \| \nabla_2 \metaloss (\bftheta) \|_2 \nonumber \\
&\le \expect_{t, \tilde{t}}^\frac{1}{2} \| \nabla_2 (\loss{val}_{\tilde{t}} \circ \nabla_1 h^*)(\bfz_t^K; \bftheta_h) - \bfh_{\tilde{t}}^K \|_2^2 + \nonumber \\
&\qquad \expect_{\tilde{t}}^\frac{1}{2} \| \bfh_{\tilde{t}}^K - \nabla_2 \metaloss_{\tilde{t}} (\bftheta) \|_2^2 + \| \nabla_2 \metaloss (\bftheta) \|_2. 
\end{align}

The first term in~\eqref{eq:app-metaloss-1stmom-h} has upper bound
\begin{align}
\label{eq:app-metaloss-1stmom-h-term1}
&\expect_{t, \tilde{t}}^\frac{1}{2} \| \nabla_2 (\loss{val}_{\tilde{t}} \circ \nabla_1 h^*)(\bfz_t^K; \bftheta_h) - \bfh_{\tilde{t}}^K \|_2^2 \nonumber \\
&\le G_{\loss{}h} \expect_{t, \tilde{t}}^\frac{1}{2} \| \bfz_t^K - \bfz_t^K \|_2^2 
\le G_{\loss{}h} \zeta \sigma
\end{align}
where the two inequalities are from Assumption~\ref{as:loss-fn}.iv) and Lemma~\ref{lemma:dist-ztk}, respectively. 

Using the chain rule~\eqref{eq:app-metagrad-h} and Lemma~\ref{lemma:sum-randvar}, the second term can be bounded through
\begin{align}
\label{eq:app-metaloss-1stmom-h-term2}
&\expect_{\tilde{t}}^\frac{1}{2} \| \bfh_{\tilde{t}}^K - \nabla_2 \metaloss_{\tilde{t}} (\bftheta) \|_2^2 = \expect_t^\frac{1}{2} \| \bfh_t^K - \nabla_2 \metaloss_t (\bftheta) \|_2^2 \nonumber \\
&\le \alpha \sum_{k=0}^{K-1} \expect_t^\frac{1}{2} \| \bfH_t^k \bfg_t^K \|_2^2 \nonumber \\
&\overset{(a)}{\le} \alpha \sum_{k=0}^{K-1} \gamma^{K-k-1} G_h G_\loss{} \expect_t^\frac{1}{2} \| \bfg_t^K \|_2^2 \nonumber \\
&= (\gamma^K - 1) \expect_t^\frac{1}{2} \| \bfg_t^K \|_2^2 \nonumber \\
&\overset{(b)}{\le} (\gamma^K - 1) \big[ \expect_t^\frac{1}{2} \| \nabla_1 (\loss{val} \circ \nabla_1 h^*)(\bfz_t^K; \bftheta_h) \|_2^2 + G_h \sqrt{T} \sigma \big] \nonumber \\
&\overset{(c)}{\le} (\gamma^K - 1) \Big[ \frac{\| \nabla_1 \metaloss (\bftheta) \|_2 + C_{\loss{}h} \sigma}{2 - \gamma^K} + G_h \sqrt{T} \sigma \Big]
\end{align}
where $(a)$ adopts~\eqref{eq:app-Htk-norm}, $(b)$ follows from the upper bound of $\| \bfg_t^K \|_2$ in~\eqref{eq:app-metaloss-2ndmom-z-term2}, and $(c)$ leverages Lemma~\ref{lemma:metaloss-2ndmom-z}. 

Plug~\eqref{eq:app-metaloss-1stmom-h-term1} and~\eqref{eq:app-metaloss-1stmom-h-term2} into~\eqref{eq:app-metaloss-1stmom-h} to arrive at
\begin{align*}
&\expect_t^\frac{1}{2} \| \nabla_2 (\loss{val} \circ \nabla_1 h^*)(\bfz_t^K; \bftheta_h) \|_2^2 \le \frac{\gamma^K-1}{2-\gamma^K} \| \nabla_1 \metaloss(\bftheta) \|_2 + \\
&\qquad \| \nabla_2 \metaloss(\bftheta) \|_2 + \Big[ (\gamma^K-1) \big( \frac{C_{\loss{}h}}{2-\gamma^K} + G_h \sqrt{T} \big) + G_{\loss{}h} \zeta \Big] \sigma \\
&= \frac{(\gamma^K-1) \| \nabla_1 \metaloss(\bftheta) \|_2 + C_{\loss{}h} \sigma}{2-\gamma^K} + \| \nabla_2 \metaloss(\bftheta) \|_2
\end{align*}
which is the desired result. 
\end{proof}

\begin{lemma}
\label{lemma:Lip-smooth-2var}
For $\forall \bftheta, \tilde{\bftheta} \in \real^D$, it holds that
\begin{align*}
\metaloss(\tilde{\bftheta}) &\le \metaloss(\bftheta) + (\tilde{\bftheta} - \bftheta)^\top \nabla \metaloss(\bftheta) + \\
&\qquad \frac{G_{\metaloss,1}}{2} \| \tilde{\bftheta}_z - \bftheta_z \|_2^2 + \frac{G_{\metaloss,2}}{2} \| \tilde{\bftheta}_h - \bftheta_h \|_2^2. 
\end{align*}
\end{lemma}

\begin{proof}
Defining $\tilde{\metaloss} (t) := \metaloss (t \tilde{\bftheta} + (1 - t) \bftheta)$, it follows from Newton-Leibniz formula that
\begin{align}
&\metaloss (\tilde{\bftheta}) = \tilde{\metaloss}(1) = \tilde{\metaloss}(0) + \int_0^1 \tilde{\metaloss}'(t) dt \nonumber \\
&= \metaloss (\bftheta) + \int_0^1 (\tilde{\bftheta} - \bftheta)^\top \nabla \metaloss (t \tilde{\bftheta} + (1 - t) \bftheta) dt \nonumber \\
&= \metaloss (\bftheta) + (\tilde{\bftheta} - \bftheta)^\top \nabla \metaloss(\bftheta) + \nonumber \\
&\qquad \int_0^1 (\tilde{\bftheta} - \bftheta)^\top \big[ \nabla \metaloss (t \tilde{\bftheta} + (1 - t) \bftheta) - \nabla \metaloss (\bftheta) \big] dt \nonumber \\
&= \metaloss (\bftheta) + (\tilde{\bftheta} - \bftheta)^\top \nabla \metaloss(\bftheta) + \nonumber \\
&\qquad \int_0^1 (\tilde{\bftheta}_z - \bftheta_z)^\top \big[ \nabla \metaloss (t \tilde{\bftheta}_z + (1 - t) \bftheta_z) - \nabla \metaloss (\bftheta_z) \big] dt + \nonumber \\
&\qquad \int_0^1 (\tilde{\bftheta}_h - \bftheta_h)^\top \big[ \nabla \metaloss (t \tilde{\bftheta}_h + (1 - t) \bftheta_h) - \nabla \metaloss (\bftheta_h) \big] dt \nonumber \\
&\overset{(a)}{\le} \metaloss (\bftheta) + (\tilde{\bftheta} - \bftheta)^\top \nabla \metaloss(\bftheta) + \nonumber \\
&\qquad \| \tilde{\bftheta}_z - \bftheta_z \|_2 \int_0^1 \| \nabla \metaloss (t \tilde{\bftheta}_z + (1 - t) \bftheta_z) - \nabla \metaloss (\bftheta_z) \|_2 dt + \nonumber \\
&\qquad \| \tilde{\bftheta}_h - \bftheta_h \|_2 \int_0^1 \| \nabla \metaloss (t \tilde{\bftheta}_h + (1 - t) \bftheta_h) - \nabla \metaloss (\bftheta_h) \|_2 dt \nonumber \\
&\overset{(b)}{\le} \metaloss(\bftheta) + (\tilde{\bftheta} - \bftheta)^\top \nabla \metaloss(\bftheta) + \\
&\qquad \frac{G_{\metaloss,1}}{2} \| \tilde{\bftheta}_z - \bftheta_z \|_2^2 + \frac{G_{\metaloss,2}}{2} \| \tilde{\bftheta}_h - \bftheta_h \|_2^2
\end{align}
where $(a)$ is from Cauchy-Schwarz inequality, and $(b)$ relies on the definition~\eqref{eq:meta-smooth} of $G_{\metaloss,j}$. 
\end{proof}

\section{Detailed numerical setups}
\label{app:hyperparams}
This section elaborates on the setups used in our numerical tests. All the hyperparameters are determined via a greedy grid search on the validation tasks. 

Following the training protocol in~\cite{MetaCurvature,WarpGrad,GAP}, the default number of convolution filters in the CNN is $128$ per block for improved expressiveness. To avoid overfitting, we reduce the filters to $64$ on 5-class 1-shot miniImageNet. The last block of WRN is a fully connected 2-layer NN of 2,048 hidden neurons, with a softmax activation appended~\cite{MetaCurvature}. 

We set $R = 60,000$,  $B = 4$, and $K = 5$ throughout the tests unless stated. The task-level learning rate $\alpha$ is $10^{-2}$ for the CNN and $2$ for the last block of WRN. To gain better numerical stability, we select SGD with $\beta_1 = 10^{-3}, \beta_2 = 10^{-4}$ and Adam~\cite{Adam} with $\beta_1 = 10^{-4}, \beta_2 = 10^{-5}$ to optimize~\eqref{eq:MetaMDA-obj-meta} on miniImageNet and tieredImageNet, respectively. The conjugate $h^*$ of DGF is constructed according to Remark~\ref{remark:valid-NN} and the last paragraph of Section~\ref{sec:h*}. As in~\cite{MetaCurvature}, the matrix $\bfP$ is not enforced to be positive semi-definite. But we restrict it to have a bounded Frobenius norm through a bounded activation $\sigma_P$. An alternative parameterization with positive semi-definiteness can be readily obtained via $\bfP = \sigma_P(\check{\bfP}) \sigma_P(\check{\bfP})^\top$. For better scalability, the weight matrix parameters $\{ (\check{\bfW}_i, \check{\bfM}_i) \}_{i=1}^I$ and $\check{\bfP}$ of $h^*$ are Kronecker-factorized, as in~\cite{MetaCurvature,MetaKFO}. For the 4-block CNN, we choose $I=3$ for $5$-class $5$-shot tieredImageNet, and $I=2$ for the rest tests. For the WRN, we set $I = 1$ for all the tests.

\vfill

\end{document}